\documentclass[10pt,twocolumn,letterpaper]{article}

\usepackage{subfigure}
\usepackage{main/cvpr}
\usepackage{times}
\usepackage{epsfig}
\usepackage{graphicx}
\usepackage{amsmath}
\usepackage{amssymb}
\usepackage{array}
\usepackage{float}
\usepackage{placeins}
\usepackage{microtype}
\usepackage{xspace}
\usepackage{subfigure}
\usepackage{booktabs}
\usepackage{eqnarray}
\usepackage{amsthm}
\usepackage{tabularx}

\cvprfinalcopy %

\def\cvprPaperID{5600} %

\ifcvprfinal\fi

\usepackage[pagebackref=true,breaklinks=true,letterpaper=true,colorlinks,bookmarks=false]{hyperref}

\usepackage{color}

\renewcommand{\etal}{{\em et~al.}}

\newenvironment{packed_item}{
\begin{itemize}
  \setlength{\itemsep}{1pt}
  \setlength{\parskip}{2pt}
  \setlength{\parsep}{0pt}
}{\end{itemize}}

\newcommand{\Cpositive}{\ensuremath{C_\mathsf{positive}}}
\newcommand{\Cnegative}{\ensuremath{C_\mathsf{negative}}}

\DeclareMathOperator*{\argmax}{arg\,max}

\def\Instagram{\textit{Instagram}\xspace}

\usepackage{color}
\usepackage{bm}

\def\Instagram{\textit{Instagram}\xspace}

\newcommand{\imagex}{\mathbf{x}}
\newcommand{\timagex}{$\imagex{}$}

\newcommand{\T}{\mathbf{T}}
\newcommand{\tT}{$\T{}$}

\newcommand{\J}{\mathbf{J}}
\newcommand{\tJ}{$\J{}$}

\newcommand{\D}{\mathbf{D}}
\newcommand{\tD}{$\D{}$\xspace}

\newcommand{\beq}{\begin{equation}}
\newcommand{\eeq}{\end{equation}}

\newcommand{\tDj}{$\D{}_\J{}$}

\definecolor{commentgrey}{rgb}{0.5,0.5,0.5}
\definecolor{graycolor}{rgb}{0.6,0.6,0.6}
\definecolor{darkergraycolor}{rgb}{0.5,0.5,0.5}

\definecolor{commentgrey}{rgb}{0.5,0.5,0.5}
\definecolor{graycolor}{rgb}{0.6,0.6,0.6}
\definecolor{darkergraycolor}{rgb}{0.5,0.5,0.5}

\begin{document}

\title{Visual Chirality}

\author{Zhiqiu Lin\textsuperscript{1} 
\qquad Jin Sun\textsuperscript{1,2} 
\qquad Abe Davis\textsuperscript{1,2} 
\qquad Noah Snavely\textsuperscript{1,2}\\
Cornell University\textsuperscript{1} \qquad Cornell Tech\textsuperscript{2}\\
}

\maketitle

\begin{abstract}
How can we tell whether an image has been mirrored? While we understand the geometry of mirror reflections very well, less has been said about how it affects distributions of imagery at scale, despite widespread use for data augmentation in computer vision. In this paper, we investigate how the statistics of visual data are changed by reflection. We refer to these changes as ``visual chirality,'' after the concept of geometric chirality---the notion of objects that are distinct from their mirror image. Our analysis of visual chirality reveals surprising results, including low-level chiral signals pervading imagery stemming from image processing in cameras, to the ability to discover visual chirality in images of people and faces. Our work has implications for data augmentation, self-supervised learning, and image forensics.
\end{abstract}
\section{Introduction}
\label{sec:Intro}

\vspace{-0.05cm}
\begin{quote}
\footnotesize{
``...there's a room you can see through the glass---that's just the same as our drawing room, only the things go the other way."}\\[3pt]
\scriptsize{\rightline{--- Lewis Carroll,\hspace{1.2cm}}\\
\rightline{``Alice's Adventures in Wonderland \& Through the Looking-Glass"\hspace{1.2cm}}}
\end{quote}
\vspace{-0.15cm}

There is a rich history of lore involving reflections. From the stories of Perseus and Narcissus in ancient Greek mythology to the adventures of Lewis Carroll's Alice and J.K. Rowling's Harry Potter, fiction is full of mirrors that symbolize windows into worlds similar to, yet somehow different from, our own. This symbolism is rooted in mathematical fact: 
what we see in reflections is consistent with a world that differs in subtle but meaningful ways from the one around us---right hands become left, text reads backward, and the blades of a fan spin in the opposite direction.
What we see is, as Alice puts it, ``just the same... only the things go the other way''. 

Geometrically, these differences can be attributed to a world where distances from the reflecting surface are negated, creating an orientation-reversing isometry with objects as we normally see them. While the properties of such isometries are well-understood in principle, much less is known about how they affect the statistics of visual data at scale. 
In other words, while we understand a great deal about how reflection changes image data, we know much less about how it changes what we learn from that data---this, despite widespread use of image reflection (e.g., mirror-flips) for data augmentation in computer vision. 

\begingroup
\setlength{\tabcolsep}{1pt} %
\renewcommand{\arraystretch}{1} %
\begin{figure}%
    \centering
    \begin{tabular}{ccc}
    \includegraphics[width=0.32\columnwidth]{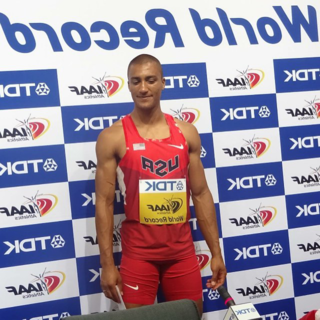} & 
    \includegraphics[width=0.32\columnwidth]{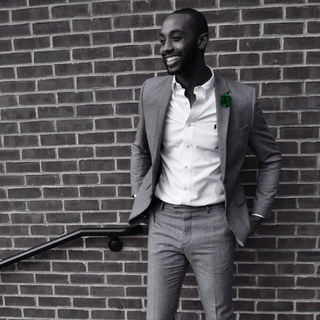} & 
    \includegraphics[width=0.32\columnwidth]{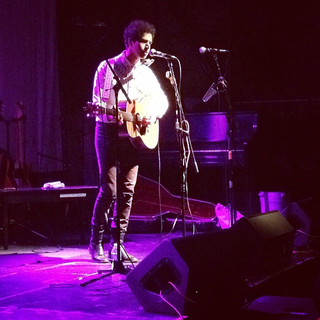} \\
    (a) & (b) & (c)\\
    \end{tabular}
    \caption{
    \textbf{Which images have been mirrored?} 
    Our goal is to understand how distributions of natural images differ from their reflections. Each of the images here appears plausible, but some subset have actually been flipped horizontally. Figuring out which can be a challenging task even for humans. Can you tell which are flipped? Answers are in Figure~\ref{fig:mirror-world-answer}.}
    \label{fig:mirror-world}%
    \vspace{-0.3cm}
\end{figure}
\endgroup

This paper is guided by a very simple question: 
How do the visual statistics of our world change when it is reflected?
One can understand some aspects of this question by considering the images in Figure~\ref{fig:mirror-world}. 
For individual objects, this question is closely related to the concept of \emph{chirality}~\cite{Kelvin-1894}. An object is said to be \emph{chiral} if it cannot be rotated and translated into alignment with its own reflection, and \emph{achiral} otherwise.\footnote{More generally, any figure is achiral if its symmetry group contains any orientation-reversing isometries.}
Put differently, we can think of chiral objects as being fundamentally changed by reflection---these are the things that ``go the other way" when viewed through a mirror---and we can think of achiral objects as simply being moved by reflection. Chirality provides some insight into our guiding question, but remains an important step removed from telling us how reflections impact learning. For this, we need a different measure of chirality---one we call \emph{visual chirality}---that describes the impact of reflection on distributions of imagery. 

In this paper, we define the notion of visual chirality, and analyze visual chirality in real world imagery, both through new theoretical tools, and through empirical analysis. Our analysis has some unexpected conclusions, including 1) deep neural networks are surprisingly good at determining whether an image is mirrored, indicating a significant degree of visual chirality in real imagery, 2) we can automatically discover high-level cues for visual chirality in imagery, including text, watches, shirt collars, face attributes, etc, and 3) we theoretically and empirically demonstrate the existence of low-level chiral cues that are imprinted in images by common image processing operations, including Bayer demosaicing and JPEG compression. These conclusions have implications in topics ranging from data augmentation to self-supervised learning and image forensics. For instance, our analysis suggests that low-level cues can reveal whether an image has been flipped, a common operation in image manipulation.

\begingroup
\setlength{\tabcolsep}{1pt} %
\renewcommand{\arraystretch}{1} %
\begin{figure}[t]
    \centering
    \begin{tabular}{ccc}
    \includegraphics[width=0.32\columnwidth]{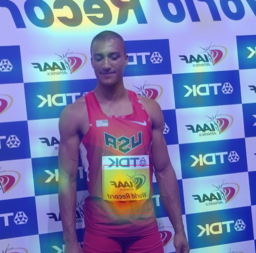} & 
    \includegraphics[width=0.32\columnwidth]{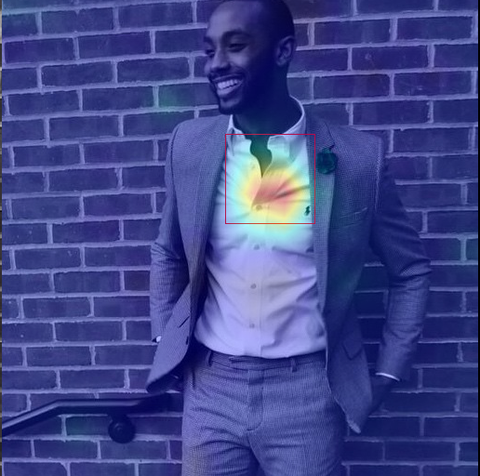} & 
    \includegraphics[width=0.32\columnwidth]{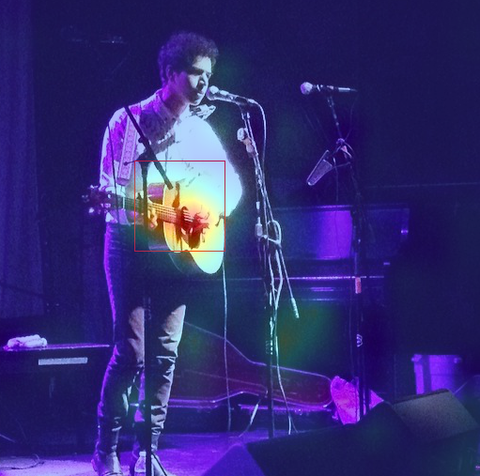} \\
    (a) & (b) & (c)\\
    \end{tabular}
    \caption{Images from Figure~\ref{fig:mirror-world} with  chirality-revealing regions highlighted. These regions are automatically found by our approach to chiral content discovery. (a, \textbf{flipped}) \textit{Text chirality.} Text (in any language) is strongly chiral. (b, \textbf{not flipped}) \textit{Object chirality.} The shirt collar, and in particular which side the buttons are on, exhibit more subtle visual chirality. (c, \textbf{flipped}) \textit{Object interaction chirality.} While guitars are often (nearly) symmetric, the way we hold them is not (the left hand is usually on the fretboard).}
    \label{fig:mirror-world-answer}%
\end{figure}

\subsection{Defining visual chirality}
\begin{figure}[t]
\centering
\includegraphics[width=0.9\columnwidth]{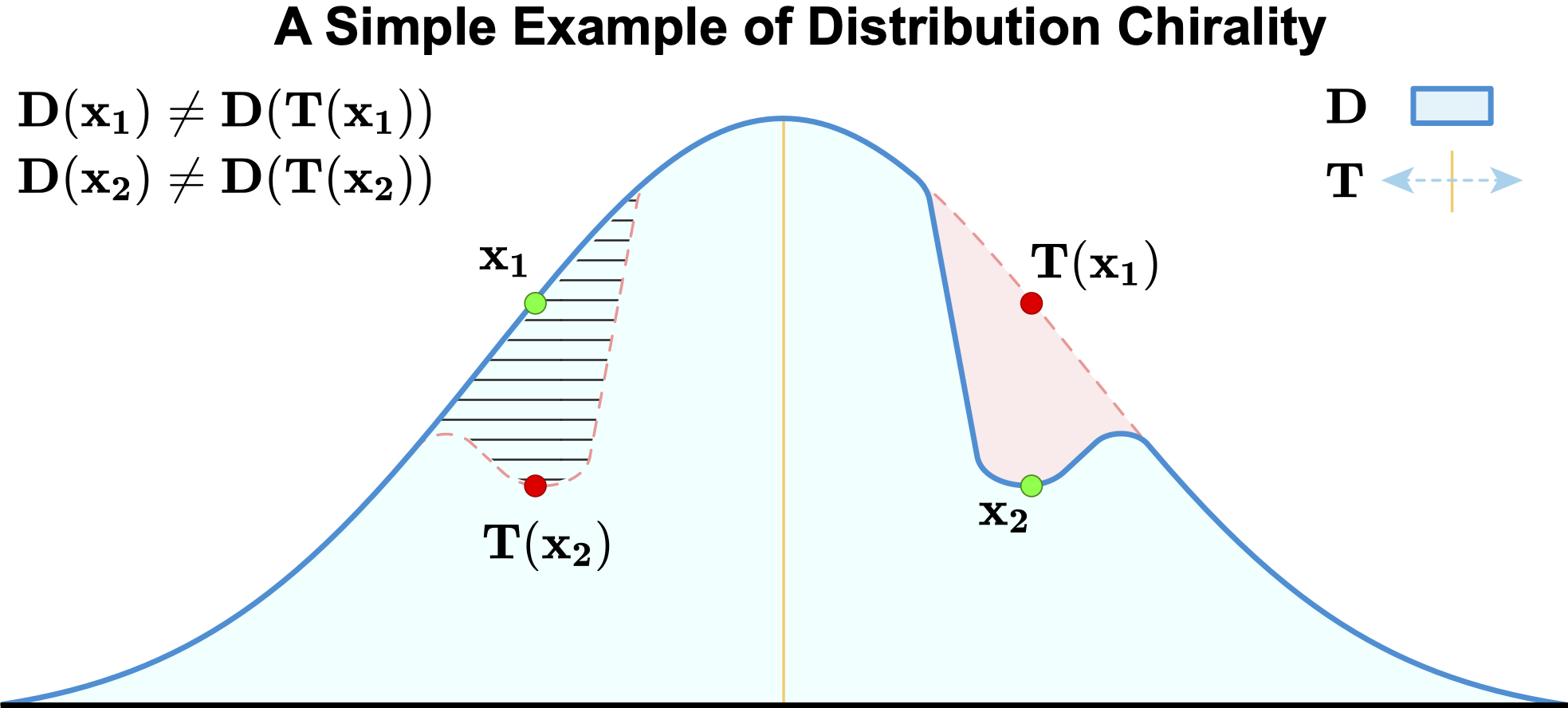}%
    \caption{The curve above represents a distribution over images (shown as a 1D distribution for simplicity). Using a transformation \tT{} to augment a sample-based approximation of the distribution \tD{} assumes symmetry with respect to \tT{}. We define visual chirality in terms of approximation error induced by this assumed symmetry when \tT{} is image reflection.}
    \label{fig:simple_chirality}%
\end{figure}

To define visual chirality, we first consider data augmentation for learning in computer vision to build intuition.
Machine learning algorithms are based on the idea that we can approximate distributions by fitting functions to samples drawn from those distributions. Viewed in this light, data augmentation can be seen as a way to improve sampling efficiency for approximating a distribution $\D{}(\imagex{})$ (where $\imagex{}$ represents data from some domain, e.g., images) by assuming that \tD{} is invariant to some transformation \tT{}. More precisely, augmenting a training sample \timagex{} with the function \tT{} assumes symmetry of the form:
\beq
\D(\imagex)=\D(\T(\imagex))
\label{eq:basic_symmetry}
\eeq
which allows us to double our effective sampling efficiency for \tD{} at the cost of approximation error wherever the assumed symmetry does not hold. This idea is illustrated in Figure \ref{fig:simple_chirality}.

Recall that for achiral objects reflection is equivalent to a change in viewpoint. Therefore, if we consider the case where \tD{} is a uniform distribution over all possible views of an object, and \tT{} is image reflection, then Equation \ref{eq:basic_symmetry} reduces to the condition for achirality. We can then define \emph{visual chirality} by generalizing this condition to arbitrary visual distributions. In other words, we define visual chirality as a measure of the approximation error associated with assuming visual distributions are symmetric under reflection. Defining visual chirality in this way highlights a close connection with data augmentation. Throughout the paper we will also see implications to a range of other topics in computer vision, including self-supervised learning and image forensics.

Note that our definition of visual chirality can also be generalized to include other transformations. In this paper we focus on reflection, but note where parts of our analysis could also apply more broadly.

\medskip
\noindent \textbf{Notes on visual chirality vs.\ geometric chirality.}
Here we make a few clarifying observations about visual chirality. First, while geometric chirality is a binary property of objects, visual chirality can be described in terms of how \emph{much} Equation \ref{eq:basic_symmetry} is violated, letting us discuss it as a continuous or relative property of visual distributions, their samples, or, as we will see in Section~\ref{sec:processing}, functions applied to visual distributions. Second, visual chirality and geometric chirality need not imply one another. For example, human hands have chiral geometry, but tend to be visually achiral because the right and left form a reflective pair and each occurs with similar frequency. Conversely, an achiral object with one plane of symmetry will be visually chiral when it is only viewed from one side of that plane. For the remainder of the paper we will refer to geometric chirality as such to avoid confusion.

\section{Related work}

Chirality has strong connections to symmetry, a long-studied topic in computer vision. Closely related to our work is recent work exploring the asymmetry of \emph{time} (referred to as ``Time's arrow'') in videos, by understanding what makes videos look like they are being played forwards or backwards~\cite{Pickup-14,Wei-18}---a sort of temporal chirality. We explore the spatial version of this question, by seeking to understand what makes images look normal or mirrored. This spatial chirality is related to other orientation problems in graphics in vision, such as detecting ``which way is up'' in an image or 3D model that might be oriented incorrectly~\cite{Vailaya-02,Fu-08}. 
Compared to upright orientation, chirality is potentially much more subtle---many images may exhibit quite weak visual chirality cues, including a couple of the images in Figure~\ref{fig:mirror-world}. Upright orientation and other related tasks have also been used as proxy tasks for unsupervised learning of feature representations~\cite{Gidaris-18}. Such tasks include the arrow of time task mentioned above~\cite{Wei-18}, solving jigsaw puzzles~\cite{Noroozi-16}, and reasoning about relative positions of image patches~\cite{Doersch-15}.

Our problem represents an interesting variation on the classic task of detecting symmetries in images~\cite{Liu-10}.
As such, our work is related to the detection and classification of asymmetric, chiral objects, as explored by Hel-Or~\etal in their work on ``how to tell left from right''~\cite{HelOr-88}, e.g., how to tell a left hand in an image from a right hand. However, this prior work generally analyzed \emph{geometric} chirality, as opposed to the \emph{visual chirality} we explore, as defined above---for instance, a right hand might be geometrically chiral but not visually chiral, while a right hand holding a pencil might visually chiral due to the prevalence of right-handed people.

Our work also relates to work on unsupervised discovery from large image collections, including work on identifying distinctive visual characteristics of cities or other image collections~\cite{Doersch-12,matzen2015bubblenet} or of yearbook photos over time~\cite{Ginosar-15}.

Finally, a specific form of chirality (sometimes referred to as \emph{cheirality}) has been explored in geometric vision. Namely, there is an asymmetry between 3D points in front of a camera and points in back of a camera. This asymmetry can be exploited in various geometric fitting tasks~\cite{Hartley-93}.

\section{Measuring visual chirality}\label{sec:model}
In principle, one way to measure visual chirality would be to densely sample a distribution and analyze symmetry in the resulting approximation. However, this approach is inefficient and in most cases unnecessary; we need not represent an entire distribution just to capture its asymmetry. Instead, we measure visual chirality by training a network to distinguish between images and their reflections. Intuitively, success at this task should be bound by the visual chirality of the distribution we are approximating.

Given a set of images sampled from a distribution, we cast our investigation of visual chirality as a simple classification task. 
Let us denote a set of training images from some distribution as $\Cpositive = \{I_1, I_2, \cdots, I_n\}$ (we assume these images are photos of the real world and have not been flipped). 
We perform a horizontal flip on each image $I_i$ to produce its reflected version $I_i'$. 
Let us denote the mirrored set as $\Cnegative = \{I_1', I_2', \cdots, I_n'\}$. 
We then assign a binary label $y_i$ to each image $I_i$ in $\Cpositive \cup \Cnegative$:
\begin{equation}
y_i = \left\{
\begin{array}{ll}
      0 & \text{if $I_i \in \Cnegative$, i.e., flipped}\\
      1 & \text{if $I_i \in \Cpositive$, i.e., non-flipped} \\
\end{array} 
\right. 
\end{equation}
We train deep Convolutional Neural Nets (CNNs) with standard classification losses for this problem, because they are 
good at learning complex distribution of natural images~\cite{Krizhevsky-12}.
Measuring a trained CNNs performance on a validation set provides insight on the visual chirality of data distribution we are investigating on.

Next we discuss details on training such a network and the techniques we use to discover the sources of visual chriality of the data distribution using a trained model as a proxy.

\medskip
\noindent \textbf{Network architecture.} We adopt a ResNet network \cite{he2016deep}, a widely used deep architecture for image classification tasks.
In particular, we use ResNet-50 and replace the last average pooling layer of the network with a global average pooling layer \cite{GAP} in order to support variable input sizes.

\medskip
\noindent \textbf{Optimization.} We train the network in a mini-batch setting using a binary cross-entropy loss. %
We optionally apply random cropping, and discuss the implications of such data augmentation below.
We normalize pixel values by per-channel mean-subtraction and dividing by the standard deviation.
We use a stochastic gradient descent optimizer~\cite{SGD} with momentum 0.9 and $L_2$ weight decay of $10^{-5}$. 

\medskip
\noindent \textbf{Hyperparameter selection.} Finding a suitable learning rate is important for this task. We perform a grid search in the log domain and select the best learning rate for each experiment by cross-validation.

\medskip
\noindent \textbf{Shared-batch training.} During training, we include both $I_i$ and $I_i'$ (i.e., positive and negative chirality versions of the same image) in the same mini-batch. We observe significant improvements in model performance using this approach, in alignment with prior self-supervised learning methods~\cite{Gidaris-18}.

\medskip

\noindent\textbf{Discovering sources of visual chirality.}
If a trained model is able to predict whether an image is flipped or not with high accuracy, it must be using a reliable set of visual features from the input image for this task.
We consider those cues as the source of visual chrility in the data distribution.

We use Class Activation Maps (CAM) \cite{zhou2016learning} as a powerful tool to visualize those discriminative regions from a trained model.
Locations with higher activation values in CAM make correspondingly larger contributions to predicting flipped images.

Throughout this paper, we visualize these activation maps as heatmaps using the Jet color scheme (red=higher activations, blue=lower activations). We only compute CAM heatmaps corresponding to an image's correct label.
Figure~\ref{fig:mirror-world-answer} shows examples of such class activation maps.

\bigskip
\noindent In the following sections, we analyze visual chirality discovered in different settings using the tools described above.

\begin{figure}[t]
\begin{center}
\begin{tabular}{cc}
\includegraphics[width=0.47\columnwidth]{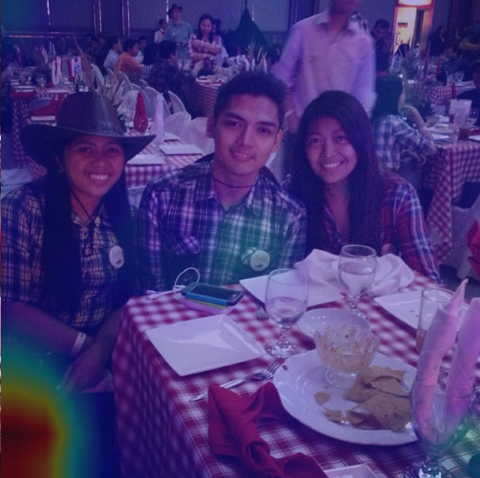} &
\includegraphics[width=0.47\columnwidth]{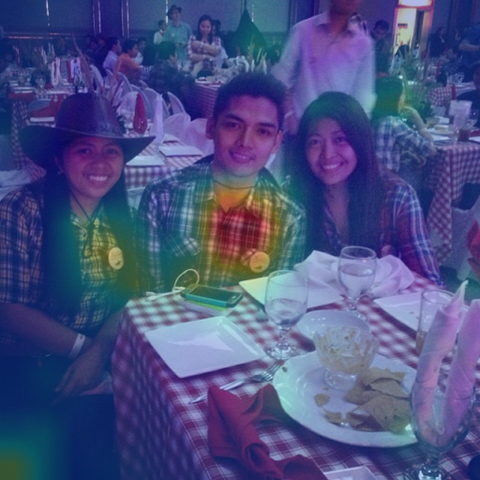} \\
(a) Resizing & (b) Random Cropping \\
\end{tabular}
\end{center}
\caption{\textbf{Resizing vs.\ random cropping as dataset preprocessing}. This figure shows CAM heatmaps for an image from models trained with two preprocessing methods: (a) resizing and (b) random cropping. We observe that the resizing scheme learns cues in the edges or corners of images (note the focus on the lower left corner of (a)), where JPEG encoding can be asymmetric. On the other hand, the random cropping scheme captures the meaningful high-level cue---the chiral shirt collar.}\label{fig:preprocess}
\end{figure}

\section{The chirality of image processing}\label{sec:processing}
When we first attempted to train our model to distinguish between images and their reflections, we quickly observed that the network would find ways to accomplish this task using low-level cues that appeared only loosely correlated with the image's content. Furthermore, the strength of these cues seemed to vary a great deal with changes in how data was prepared. For example, Figure \ref{fig:preprocess} shows two different CAM heatmaps for the same sample image. The left is derived from a network trained on resized data, and the right is derived from a network trained on random crops of the same data. Both maps identify a dark corner of the image as being discriminative, as well as part of the shirt on one of the the image's human subjects. However, these networks appear to disagree about the relative strength of the chiral cues in these regions. This result illustrates how the way we capture and process visual data---even down to the level of Bayer mosaics in cameras or JPEG compression---can have a significant impact on its chirality. In this section and the supplemental material we develop theoretical tools to help reason about that impact, and use this theory to predict what networks will learn in experiments.

\subsection{Transformation commutativity}
The key challenge of predicting how an imaging process will affect chirality is finding a way to reason about its behavior under minimal assumptions about the distribution to which it will be applied. For this, we consider what it means for an arbitrary imaging transformation \tJ{} to preserve the symmetry of a distribution \tD{} (satisfying Equation \ref{eq:basic_symmetry}) under a transformation \tT{}. There are two ways we can define this. The first is simply to say that if some symmetry exists in the distribution \tD{} then the same symmetry should exist in \tDj{}, the transformation of that distribution by \tJ{}. The second is to say that if elements $\imagex_a$ and $\imagex_b$ are related by $\imagex_b=\T\imagex_a$, then this relationship should be preserved by \tJ, meaning $\J\imagex_b=\T\J\imagex_a$. In our supplemental material we show that both definitions hold when $\J$ commutes with $\T$, and that the second definition does not hold when \tJ{} does not commute with \tT{}. With these observations, commutativity becomes a powerful tool for predicting how a given process \tJ{} can affect chirality.

\subsection{Predicting chirality with commutativity}
In our supplemental material we provide a detailed analysis of the commutativity of Bayer demosaicing, JPEG compression, demosaicing + JPEG compression, and all three of these again combined with random cropping. We then show that, in all six cases, our analysis of commutativity predicts the performance of a network trained from scratch to distinguish between random noise images and their reflection. These predictions also explain our observations in Figure \ref{fig:preprocess}. While the full details are presented in the supplemental material, some key highlights include:
\begin{packed_item}
  \item Demosaicing and JPEG compression are both individually chiral and chiral when combined.
  \item When random cropping is added to demosaicing or JPEG compression individually, they become achiral.
  \item When demosaicing, JPEG compression, and random cropping are all combined, the result is chiral.
\end{packed_item}
This last conclusion is especially surprising---it implies that common image processing operations inside our cameras may leave a \emph{chiral imprint}, i.e., that they imprint chiral cues that are imperceptible to people, but potentially detectable by neural networks, and that these features are robust to random cropping.
Thus, these conclusions have implications on image forensics. For instance, our analysis gives us new theoretical and practical tools for determining if image content has been flipped, a common operation in image editing.

Finally, our analysis of how commutativity relates to the preservation of symmetries 
makes only very general assumptions about \tJ{}, \tT{}, and \tD{}, making it applicable to more arbitrary symmetries. For example, Doersch~\etal~\cite{Doersch-15} found that when they used the relative position of different regions in an image as a signal for self-supervised learning, the networks would ``cheat” by utilizing chromatic aberration for prediction. Identifying the relative position of image patches requires asymmetry with respect to image translation. Applied to their case, our analysis is able to predict that chromatic aberration, which does not commute with translation, can provide this asymmetry.

\section{High-level visual chirality}\label{sec:highlevel}

While analysis of chiralities that arise in image processing have useful implications in forensics, we are also interested in understanding what kinds of high-level visual content (objects, object regions, etc.) reveals visual chirality, and whether we can discover these cues automatically. 
As described in Section~\ref{sec:processing}, if we try to train a network from scratch, it invariably starts to pick up on uninterpretable, low-level image signals. 
Instead, we hypothesize that if we start with a ResNet network that has been pre-trained on ImageNet object classification, then it will have a 
familiarity with objects that will allow it to avoid picking up on low-level cues. Note, that such ImageNet-trained networks should \emph{not} have features sensitive to specifically to chirality---indeed, as noted above, many ImageNet classifiers are trained using random horizontal flips as a form of data augmentation.

\medskip
\noindent \textbf{Data.} What distribution of images do we use for training? We could try to sample from the space of all natural images. 
However, because we speculate that many chirality cues have to do with people, and with manmade objects and scenes, we start with images that feature \emph{people}. In particular, we utilize the StreetStyle dataset of Matzen~\etal~\cite{matzen17streetstyle}, which consists of millions of images of people gathered from Instagram. 
For our work, we select a random subset of 700K images from StreetStyle, and refer to this as the \Instagram dataset; example images are shown in Figures~\ref{fig:mirror-world} and \ref{fig:instaclusters}. We randomly sample 5K images as a test set $S_\mathsf{test}$, and split the remaining images into training and validation sets with a ratio of 9:1 (unless otherwise stated, we use this same train/val/test split strategy for all experiments in this paper).

\medskip
\noindent \textbf{Training.} We trained the chirality classification approach described in Section~\ref{sec:model} on \Instagram, starting from an ImageNet-pretrained model. As it turns out, the transformations applied to images before feeding them to a network are crucial to consider. Initially, we downsampled all input images bilinearly to a resolution of 512$\times$512. A network so trained achieves a \textbf{$92\%$} accuracy on the \Instagram test set, a surprising result given that determining whether an image has been flipped can be difficult even for humans.

As discussed above, it turns out that our networks were still picking up on traces left by low-level processing, such as boundary artifacts produced by JPEG encoding, as evidenced by CAM heatmaps that often fired near the corners of images. 
In addition to pre-training on ImageNet, we found that networks can be made more resistant to the most obvious such artifacts by performing random cropping of input images. 
In particular, we randomly crop a 512$\times$512 window from the input images during training and testing (rather than simply resizing the images).  A network trained in such a way still achieves a test accuracy to 80\%, still a surprisingly high result.

\begin{table}[t]
\begin{center}
\begin{tabular}{llcc}
\toprule 
{Training set} & Preprocessing & \multicolumn{2}{c}{{Test Accuracy}} \\
\cmidrule{3-4}
&  &  \Instagram & F100M \\
 \midrule
\Instagram & Resizing &  0.92 & 0.57\\
\Instagram & RandCrop & 0.80 & 0.59\\
\Instagram (no-text) & RandCrop &  0.74 & 0.55 \\
\bottomrule
\end{tabular}
\end{center}
\caption{\textbf{Chirality classification performance of models trained on \Instagram.} Hyper-parameters were selected by cross validation. The first column indicates the training dataset, and the second column the processing that takes place on input images. The last columns report on a held-out test set, and on an unseen dataset (Flickr100M, or F100M for short). Note that the same preprocessing scheme (resize vs.\ random crop) is applied to both the training and test sets, and the model trained on \Instagram without text is also tested on \Instagram without text.}
\label{tab:object}
\end{table}

\begin{figure*}[htp!]
\centering
\resizebox{\textwidth}{!}{
\begin{tabular}{m{0.25cm} m{2.62cm} m{2.62cm} m{2.62cm} m{2.62cm} m{2.62cm}}
\rotatebox{90}{Smartphones} &
\includegraphics[width=2.7cm]{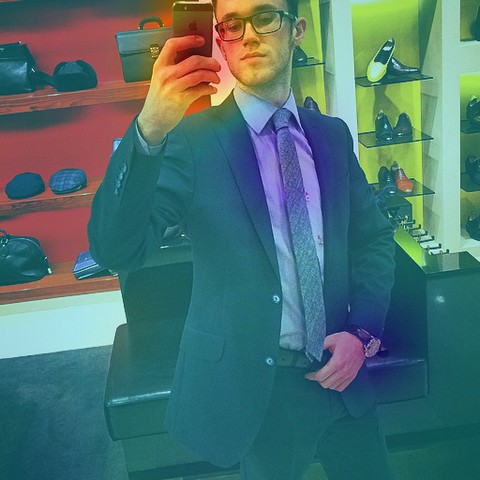} &
\includegraphics[width=2.7cm]{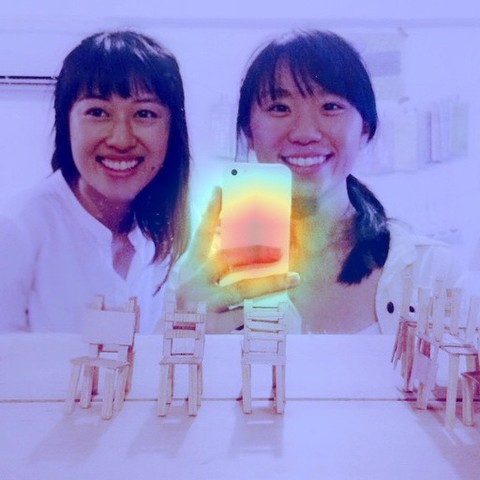} &
\includegraphics[width=2.7cm]{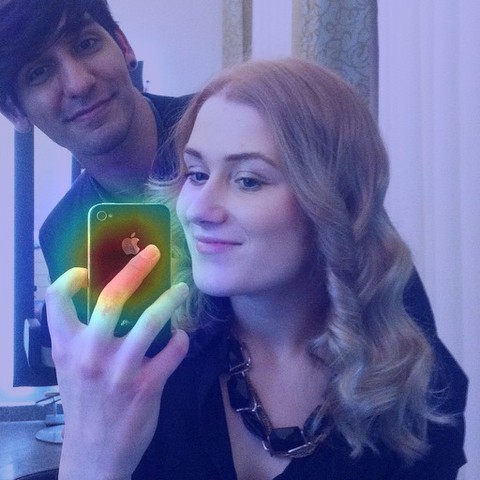} &
\includegraphics[width=2.7cm]{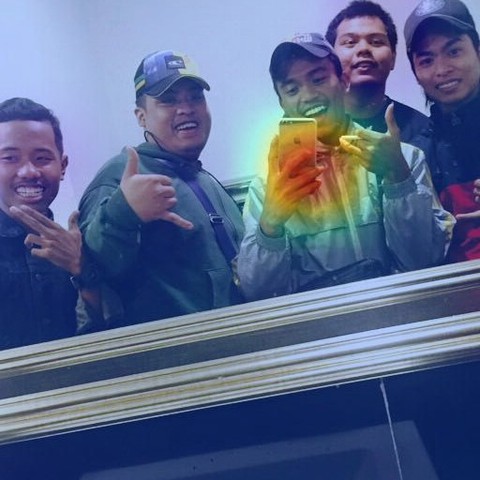} &
\includegraphics[width=2.7cm]{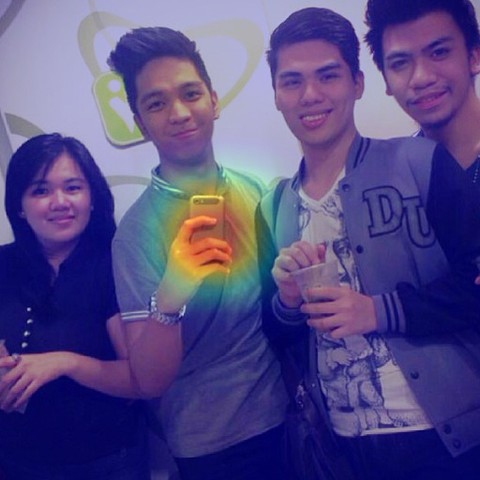} \\
\rotatebox{90}{Guitars} &
\includegraphics[width=2.7cm]{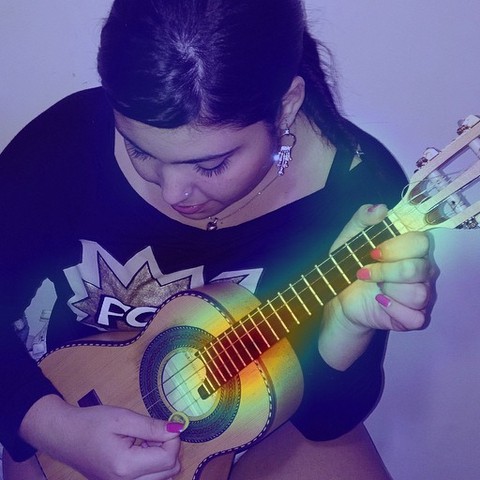} &
\includegraphics[width=2.7cm]{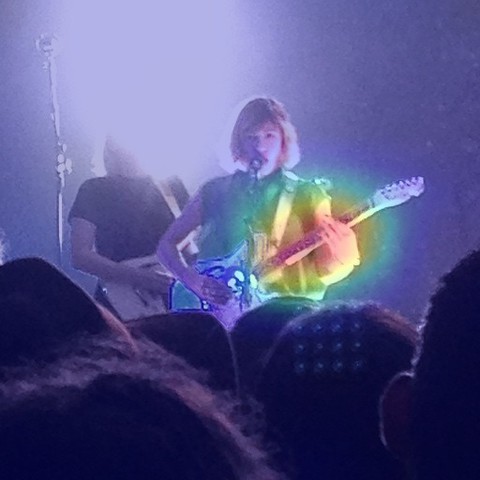} &
\includegraphics[width=2.7cm]{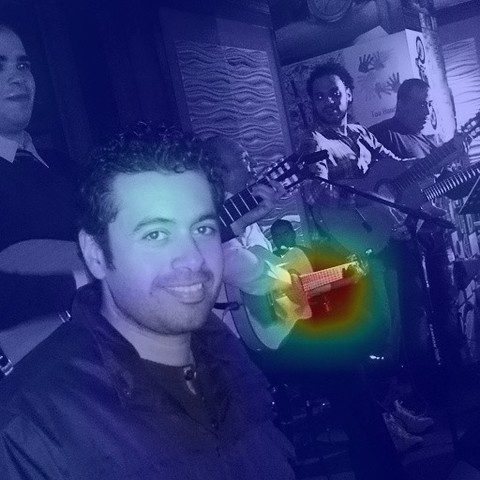} &
\includegraphics[width=2.7cm]{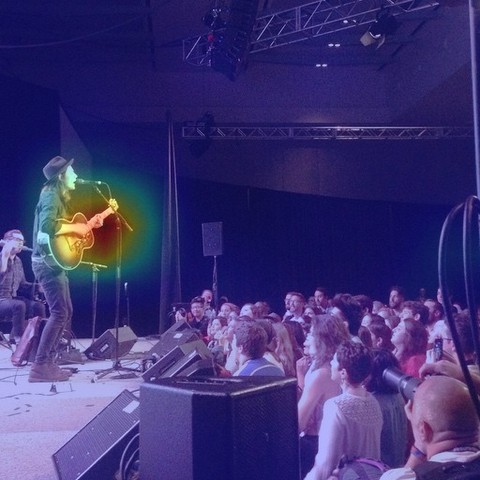} &
\includegraphics[width=2.7cm]{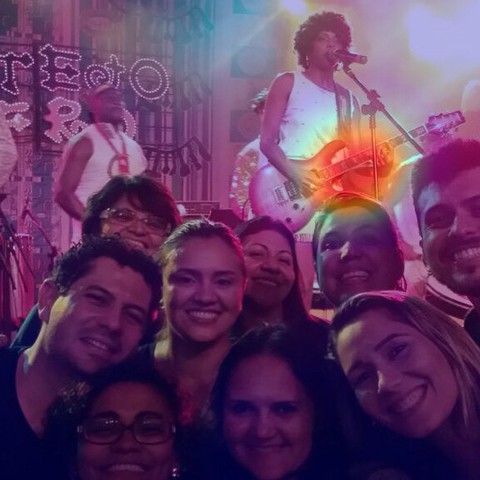} \\
\rotatebox{90}{Watches} &
\includegraphics[width=2.7cm]{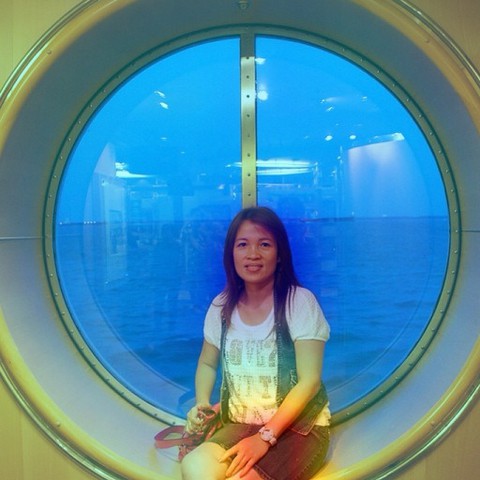} &
\includegraphics[width=2.7cm]{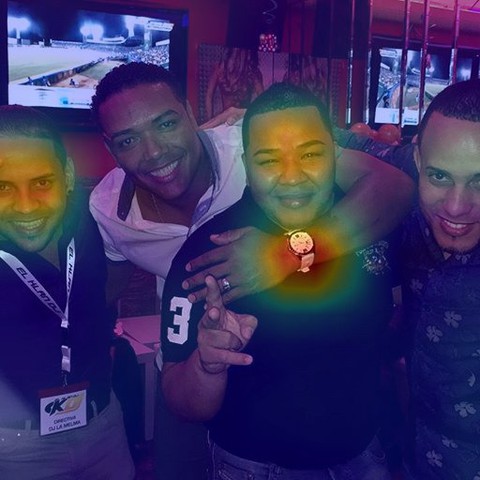} &
\includegraphics[width=2.7cm]{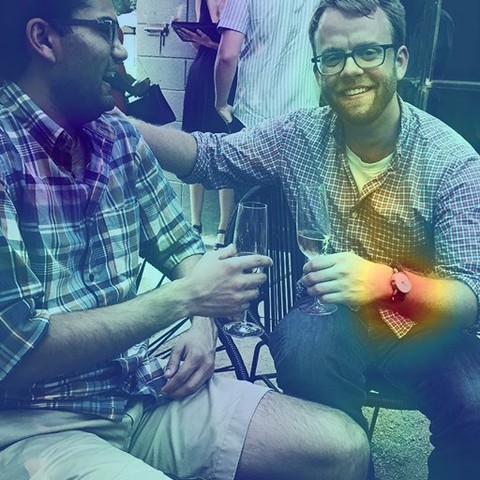} &
\includegraphics[width=2.7cm]{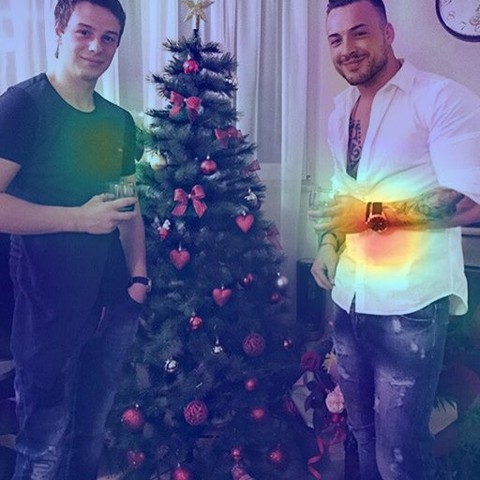} &
\includegraphics[width=2.7cm]{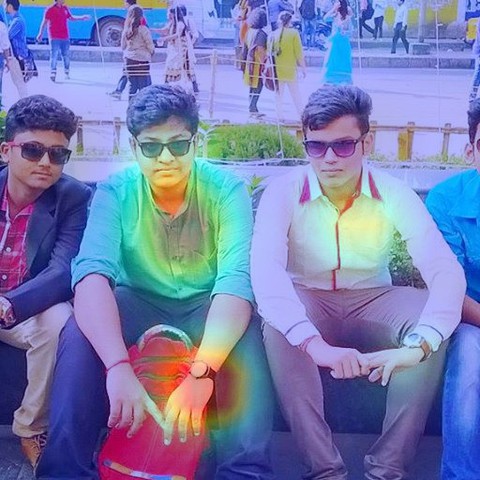} \\
\rotatebox{90}{Shirt collars} &
\includegraphics[width=2.7cm]{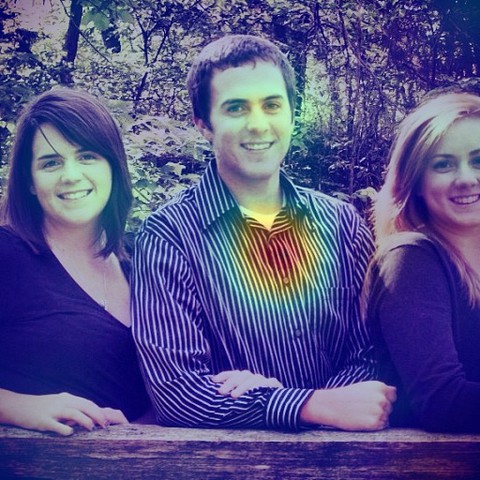} &
\includegraphics[width=2.7cm]{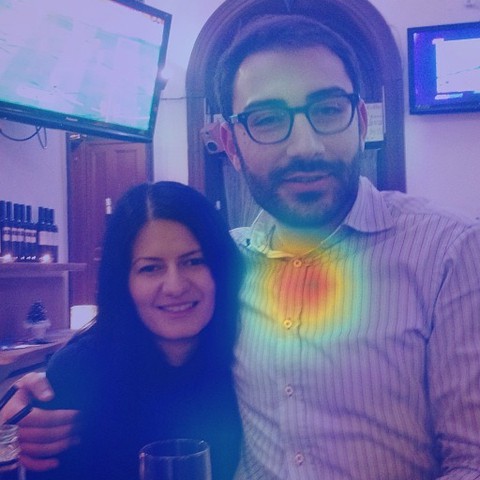} &
\includegraphics[width=2.7cm]{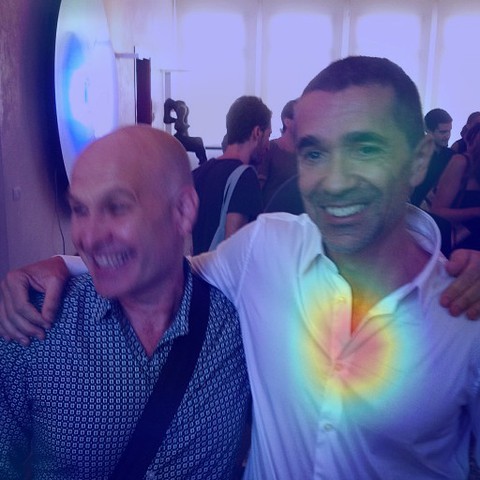} &
\includegraphics[width=2.7cm]{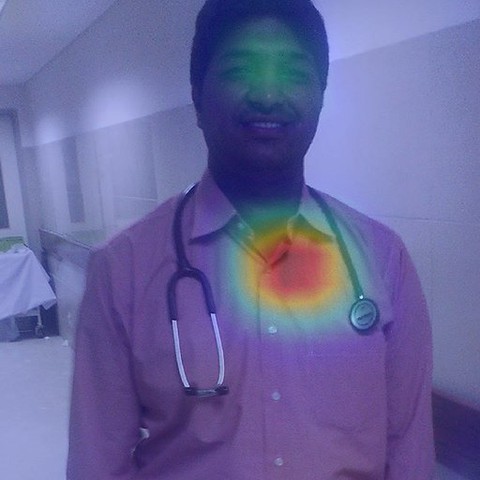} &
\includegraphics[width=2.7cm]{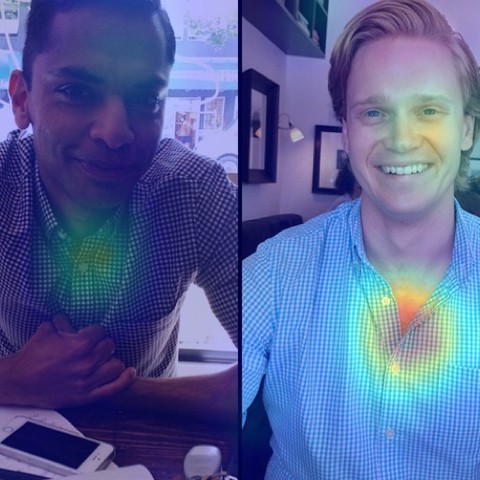} \\
\rotatebox{90}{Shirt pockets} &
\includegraphics[width=2.7cm]{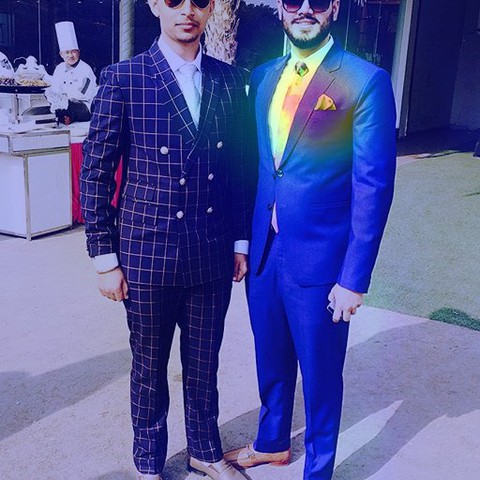} &
\includegraphics[width=2.7cm]{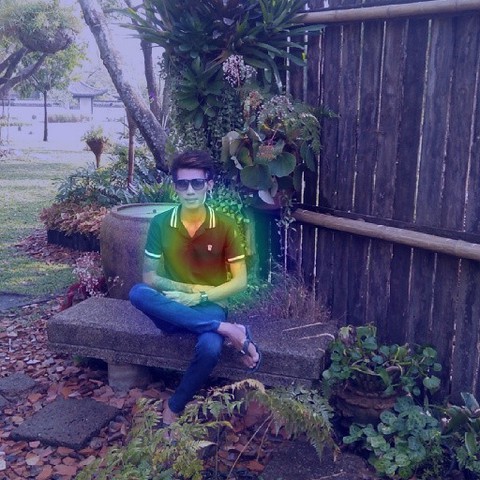} &
\includegraphics[width=2.7cm]{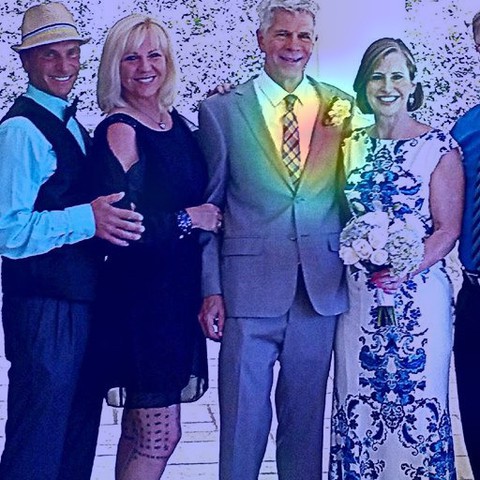} &
\includegraphics[width=2.7cm]{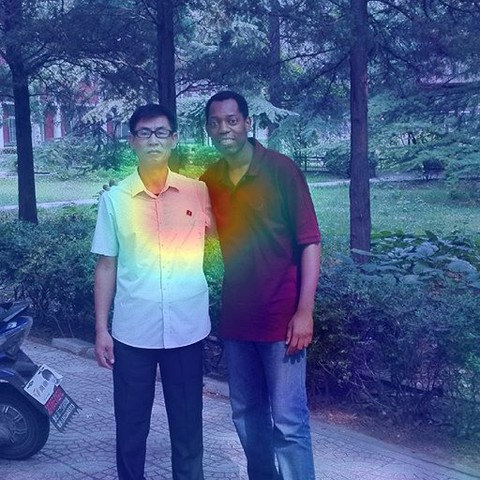} &
\includegraphics[width=2.7cm]{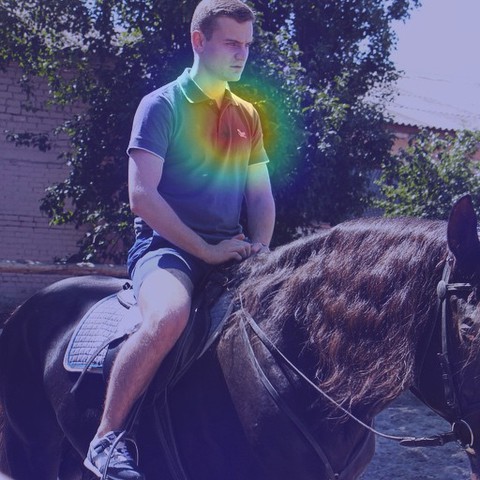} \\
\rotatebox{90}{Faces} &
\includegraphics[width=2.7cm]{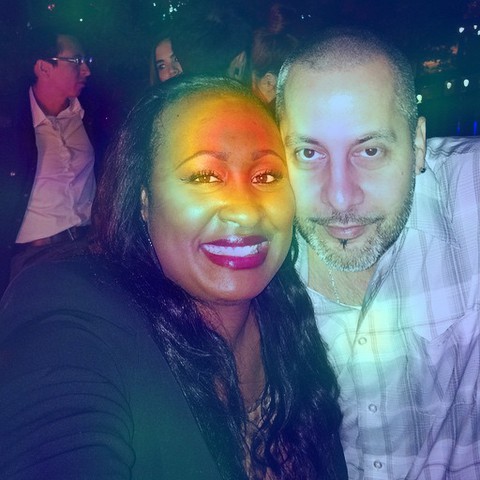} &
\includegraphics[width=2.7cm]{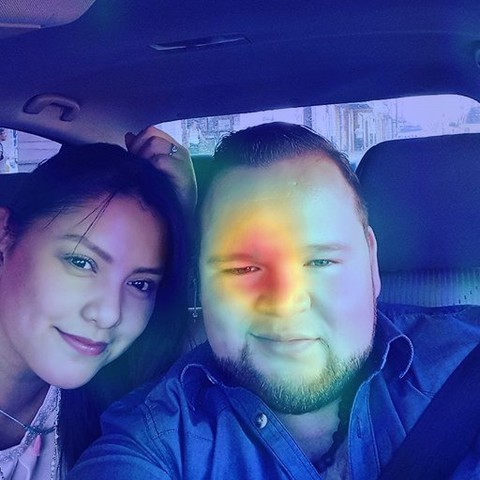} &
\includegraphics[width=2.7cm]{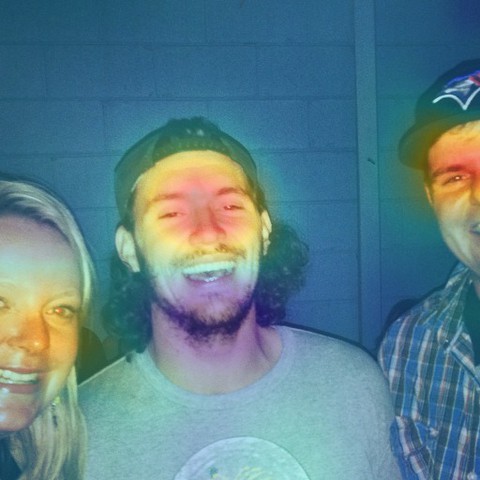} &
\includegraphics[width=2.7cm]{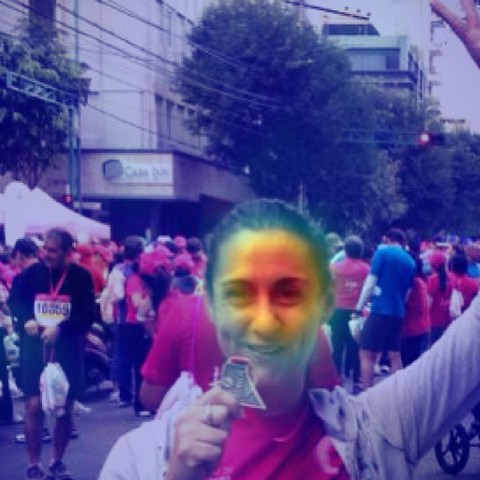} &
\includegraphics[width=2.7cm]{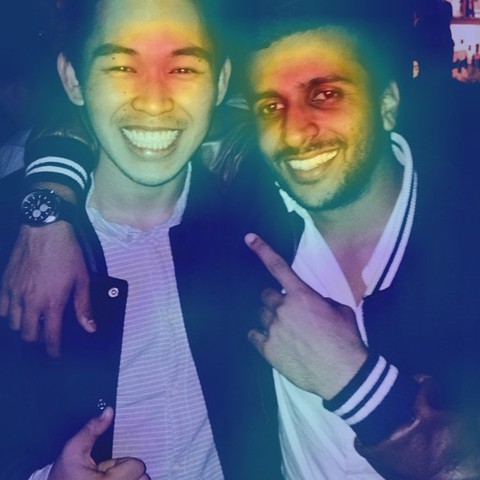} \\
\end{tabular}
}
\caption{\textbf{Chiral clusters discovered in the Instagram dataset.} Each row shows selected images from a single discovered cluster. Each image is shown with its corresponding CAM heatmap superimposed, where red regions are highly correlated with its true chirality. We discover a range of object-level chiral clusters, such as cellphones, watches, and shirts.}\label{fig:instaclusters}
\end{figure*}

\medskip
\noindent \textbf{Non-text cues.} Examining the most confident classifications, we found that many involved text (e.g., on clothing or in the background), and that CAM heatmaps often predominantly focused on text regions. Indeed, text is such a strong signal for chirality that it seems to drown out other signals. 
This yields a useful insight: we may be able to leverage chirality to learn a text detector via self-supervision, for any language (so long as the writing is chiral, which is true for many if not all languages).

However, for the purpose of the current analysis, we wish to discover non-text chiral cues as well. To make it easier to identify 
such cues, we ran an automatic text detector~\cite{EAST} on \Instagram, split it into text and no-text subsets, and then randomly sampled %
the no-text subset to form new training and text set. On the no-text subset, chirality classification accuracy drops from 80\% to 74\%---lower, but still well above chance.

\medskip
\noindent \textbf{Generalization.} 
Perhaps our classifier learns features specific to Instagram images. To test this, Table~\ref{tab:object} (last column) shows the evaluation accuracy of all models (without fine-tuning) on another dataset of Internet photos, a randomly selected subset of photos from Flickr100M~\cite{thomee2016flickr}. Note that there is a significant domain gap between \Instagram and Flickr100M, in that images in our \Instagram dataset all contain people, whereas Flickr100M features more general content (landscapes, macro shots, etc.) in addition al people. While the performance on Flickr100M is naturally lower than on \Instagram, our \Instagram-trained models still perform above chance rates, 
with an accuracy of 55\% (or 59\% if text is considered), suggesting that our learned chiral features can generalize to new distributions of photos.

\subsection{Revealing object-level chiral features}\label{sec:discovering} 

Inspecting the CAM heatmaps derived from our non-text-trained \Instagram model reveals a network that focuses on a coherent set of local regions, such as smart phones and shirt pockets, across different photos. 
To further understand what the network has learned, we develop a way to group the images, as well as their CAM heatmaps, to determine which cues are most common and salient.
Inspired by work on mid-level discriminative patch mining~\cite{Doersch-12,Singh-12,li2015mid,matzen2015bubblenet}, we propose 
a method built upon CAM that we call \emph{chiral feature clustering}, which  automatically groups images based on the similarity of features extracted by the network, in regions deemed salient by CAM.

\medskip
\noindent \textbf{Chiral feature clustering.}
First, we extract the most discriminative local chiral feature from each image to use as input to our clustering stage. To do so, we consider the feature maps that are output from the last convolutional layer of our network. As is typical of CNNs, these features are maps with low spatial resolution, but with high channel dimensionality (e.g., 2048). 

Given an input image, let us denote the output of this last convolutional layer as $\mathbf{f}$, which in our case is a feature map of dimensions $16 \times 16 \times 2048$ ($w \times h \times c$). Let $\mathbf{f}(x,y)$ denote the 2048-D vector at location $(x,y)$ of $\mathbf{f}$. We apply CAM, using the correct chirality label for the image, to obtain a $16 \times 16$ weight activation map $A$. Recall that the higher the value of $A(x,y)$, the higher the contribution of the local region corresponding to $(x,y)$ to the prediction of the correct chirality label. 

We then locate the spatial maxima of $A$, $(x^*,y^*) = \argmax_{(x,y)} A(x,y)$ in each image. These correspond to points deemed maximally salient for the chirality task by the network. We extract $\mathbf{f}(x^*, y^*)$ as a local feature vector describing this maximally chiral region. Running this procedure for each image yields a collection for feature vectors, on which we run $k$-means clustering.

\medskip
\noindent \textbf{Results of chiral feature clustering.}
We apply this clustering procedure to our no-text \Instagram test set, using $k=500$ clusters. We observe that this method is surprisingly effective and identifies a number of intriguing object-level chiral cues in our datasets. We refer to these clusters as \emph{chiral clusters}. Examples of striking high-level chiral clusters are shown in Figure~\ref{fig:instaclusters}, and include phones (e.g., held in a specific way to take photos in a mirror), watches (typically worn on the left hand), shirt collars (shirts with buttoned collared typically button on a consistent side), shirt pockets, pants, and other objects.

Many of these discovered chiral clusters are highly interpretable. However, some clusters are difficult to understand. For instance, in the face cluster shown in the last row of Figure~\ref{fig:instaclusters}, the authors could not find obvious evidence of visual chirality, leading us to suspect that there may be subtle chirality cues in faces. We explore this possibility in Section~\ref{sec:faces}. 
We also observe that some clusters focus on sharp edges in the image, leading us to suspect that some low-level image processing cues are being learned in spite of the ImageNet initialization and random cropping.

\section{Visual chirality in faces}\label{sec:faces}

\begin{figure*}[ht!]
\centering
\resizebox{\textwidth}{!}{
\begin{tabular}{m{0.25cm} m{2.62cm} m{2.62cm} m{2.62cm} m{2.62cm} m{2.62cm}}
\rotatebox{90}{Hair part} &
\includegraphics[width=2.7cm]{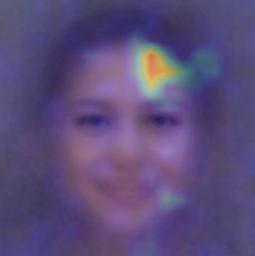} &
\includegraphics[width=2.7cm]{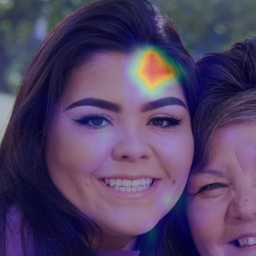} &
\includegraphics[width=2.7cm]{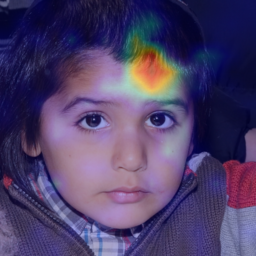} &
\includegraphics[width=2.7cm]{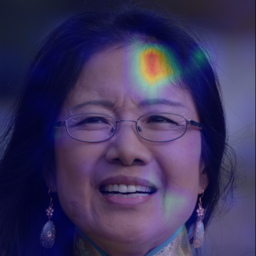} &
\includegraphics[width=2.7cm]{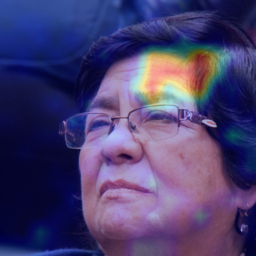} \\
\rotatebox{90}{Eyes} &
\includegraphics[width=2.7cm]{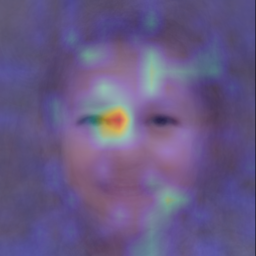} &
\includegraphics[width=2.7cm]{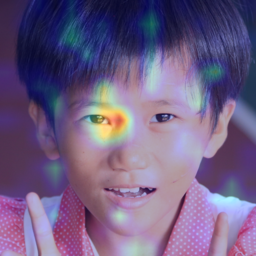} &
\includegraphics[width=2.7cm]{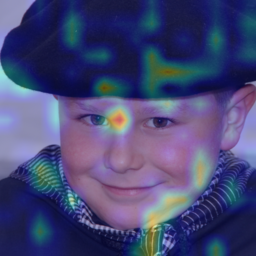} &
\includegraphics[width=2.7cm]{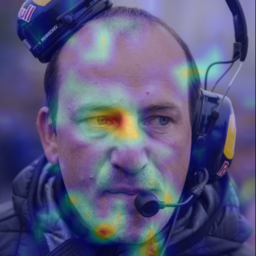} &
\includegraphics[width=2.7cm]{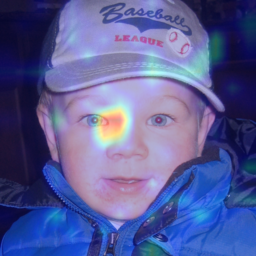} \\
\rotatebox{90}{Beard} &
\includegraphics[width=2.7cm]{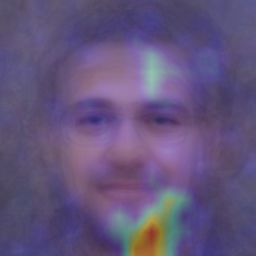} &
\includegraphics[width=2.7cm]{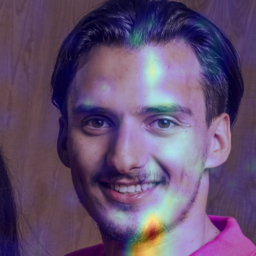} &
\includegraphics[width=2.7cm]{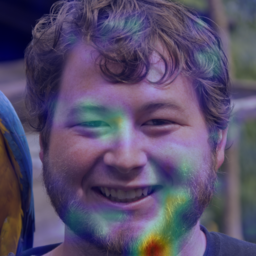} &
\includegraphics[width=2.7cm]{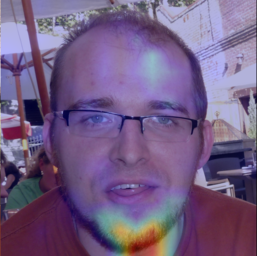} &
\includegraphics[width=2.7cm]{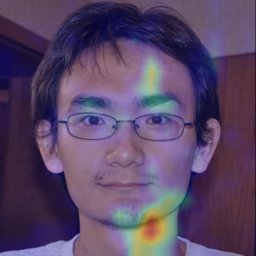} \\
\end{tabular}
}
\caption{\textbf{Chiral clusters found in FFHQ.} It shows 3 chiral clusters of FFHQ dataset. The leftmost image of each row is the average face + CAM heatmap for all non-flipped images inside the each cluster. We also show some random non-flipped examples for each cluster.}\label{fig:faceclusters}
\end{figure*}

Inspired by our results on the Instagram dataset in Section~\ref{sec:highlevel}, we now analyze chirality in face images.
To do so, we use the FFHQ dataset~\cite{karras2018style} as the basis for learning. FFHQ is a recent dataset of 70K high-quality faces introduced in the context of training generative methods. We use 7\% of the images as a test set and the remaining images for training and validation.
We train various models 
on FFHQ, first downsampling images to a resolution of 520$\times$520, then randomly cropping to 512$\times$512.  We train a standard model starting from ImageNet pre-trained features.  This model achieves an accuracy of 81\%, which is a promising indicator that our network can indeed learn to predict the chirality of faces with accuracy significantly better than chance.

However, perhaps there is some bias in FFHQ that leads to spurious chirality signals. For instance, since a face detector is used to create FFHQ, there is the possibility that the detector is biased, e.g., for left-facing faces vs.\  right-facing faces. To test this, we evaluate how well our FFHQ-trained model generalizes to other independent datasets. In particular, we evaluate this model (without fine-tuning) on another dataset, LFW, a standard face dataset~\cite{LFWTech}. We upsample the low-resolution images in LFW to 512$\times$512 to match our input resolution. This yields an accuracy of 60\%---not as high as FFHQ, perhaps due to different distributions of faces, but still significantly better than chance.

To qualitatively explore the chirality cues the model has identified, we show a sample of chiral clusters derived from the FFHQ test set in Figure~\ref{fig:faceclusters}. We can see that the CAM heatmaps in each cluster focus on specific facial regions. 
Based on these cluster, we have identified some intriguing preliminary hypotheses about facial chirality:

\medskip
\noindent \textbf{Hair part}. The first cluster in Figure~\ref{fig:faceclusters} indicates a region around the part of the hair on the left side of the forehead. We conjecture that this could be due to bias in hair part direction. We manually inspected a subset of the FFHQ test set, and found that a majority of people pictured parted their hair from left to right (the ratio is $\sim$2:1 for photos with visible hair part), indicating a bias for asymmetry in hair, possibly due to people preferentially using their dominant right hand to part their hair.

\medskip
\noindent \textbf{Predominant gaze direction}, aka ocular dominance\footnote{\text{https://en.wikipedia.org/wiki/Ocular\_dominance}}. 
The second cluster cluster in Figure~\ref{fig:faceclusters} highlights a region around the corner of the right eye. We conjectured that this may have to do with bias in gaze direction, possibly due to ocular dominance. We use gaze detection software\footnote{https://github.com/shaoanlu/GazeML-keras} to determine
and compare the locations of the pupil in the left and right eyes. We found that indeed more than two thirds of people in portrait photographs gaze more towards the left.

\smallskip
Note that there are also other clusters left to be explained (for example the ``beard'' cluster, which may perhaps be due to males tending to use right hands to shave or groom their beard). Exploring such cues would make for interesting future work and perhaps reveal interesting asymmetries in our world.

\section{Conclusion}
We propose to discover visual chirality in image distributions using a self-supervised learning approach by predicting whether a photo is flipped or not, and by analyzing properties of transformations that yield chirality.
We report various visual chirality cues identified using our tool on a variety of datasets such as Instagram photos and FFHQ face images.
We also find that low-level chiral cues are likely pervasive in images, due to chiralities inherent in standard image processing pipelines.
Our analysis has implications in data augmentation, self-supervised learning, and image forensics.
Our results implies that visual chirality indeed exists in many vision datasets and such properties should be taken into account when developing real-world vision systems.
However, our work suggests that it can also be used as a signal that can be leveraged in interesting new ways. For instance, since text is highly chiral, our work points to interesting future direction in utilizing chirality in a self-supervised way to learn to detect text in images in the wild. We hope that our work will also inspire further investigation into subtle biases imprinted our visual world.

\medskip
\noindent \textbf{Acknowledgements.} This research was supported in part by the generosity of Eric and Wendy Schmidt by recommendation of the Schmidt Futures program.

{\small
\bibliographystyle{main/ieee}
\bibliography{ms}

\begin{thebibliography}{10}\itemsep=-1pt

\bibitem{SGD}
L.~Bottou.
\newblock Stochastic gradient descent tricks.
\newblock In {\em Neural networks: Tricks of the trade. Springer}, page
  421–436, 2012.

\bibitem{Doersch-15}
C.~Doersch, A.~Gupta, and A.~A. Efros.
\newblock Unsupervised visual representation learning by context prediction.
\newblock {\em ICCV}, 2015.

\bibitem{Doersch-12}
C.~Doersch, S.~Singh, A.~Gupta, J.~Sivic, and A.~A. Efros.
\newblock What makes {P}aris look like {P}aris?
\newblock {\em SIGGRAPH}, 31(4), 2012.

\bibitem{Fu-08}
H.~Fu, D.~Cohen-Or, G.~Dror, and A.~Sheffer.
\newblock Upright orientation of man-made objects.
\newblock In {\em SIGGRAPH}, 2008.

\bibitem{Gidaris-18}
S.~Gidaris, P.~Singh, and N.~Komodakis.
\newblock Unsupervised representation learning by predicting image rotations.
\newblock In {\em ICLR}, 2018.

\bibitem{Ginosar-15}
S.~Ginosar, K.~Rakelly, S.~Sachs, B.~Yin, and A.~A. Efros.
\newblock A {C}entury of {P}ortraits: {A} visual historical record of american
  high school yearbooks.
\newblock In {\em ICCV Workshops}, December 2015.

\bibitem{Hartley-93}
R.~Hartley.
\newblock Cheirality invariants.
\newblock In {\em Proc. DARPA Image Understanding Workshop}, 1993.

\bibitem{he2016deep}
K.~He, X.~Zhang, S.~Ren, and J.~Sun.
\newblock Deep residual learning for image recognition.
\newblock In {\em CVPR}, pages 770--778, 2016.

\bibitem{HelOr-88}
Y.~Hel-Or, S.~Peleg, and H.~Hel-Or.
\newblock How to tell right from left.
\newblock In {\em CVPR}, 1988.

\bibitem{LFWTech}
G.~B. Huang, M.~Ramesh, T.~Berg, and E.~Learned-Miller.
\newblock Labeled faces in the wild: A database for studying face recognition
  in unconstrained environments.
\newblock Technical Report 07-49, University of Massachusetts, Amherst, October
  2007.

\bibitem{karras2018style}
T.~Karras, S.~Laine, and T.~Aila.
\newblock A style-based generator architecture for generative adversarial
  networks.
\newblock {\em CoRR}, abs/1812.04948, 2018.

\bibitem{Kelvin-1894}
W.~T. Kelvin.
\newblock The molecular tactics of a crystal.
\newblock {\em J. Oxford Univ. Jr. Sci. Club}, 18:3–57, 1894.

\bibitem{Krizhevsky-12}
A.~Krizhevsky, I.~Sutskever, and G.~E. Hinton.
\newblock Image{N}et classification with deep convolutional neural networks.
\newblock In {\em NeurIPS}, 2012.

\bibitem{li2015mid}
Y.~Li, L.~Liu, C.~Shen, and A.~van~den Hengel.
\newblock Mid-level deep pattern mining.
\newblock In {\em CVPR}, 2015.

\bibitem{Liu-10}
Y.~Liu, H.~Hel-Or, C.~S. Kaplan, and L.~J.~V. Gool.
\newblock Computational symmetry in computer vision and computer graphics.
\newblock {\em Foundations and Trends in Computer Graphics and Vision},
  5(1-2):1--195, 2010.

\bibitem{GAP}
Q.~C. M.~Lin and S.~Yan.
\newblock Network in network.
\newblock In {\em International Conference on Learning Representations}, pages
  2921--2929, 2014.

\bibitem{matzen17streetstyle}
K.~Matzen, K.~Bala, and N.~Snavely.
\newblock Street{S}tyle: {E}xploring world-wide clothing styles from millions
  of photos.
\newblock {\em CoRR}, abs/1706.01869, 2017.

\bibitem{matzen2015bubblenet}
K.~Matzen and N.~Snavely.
\newblock {BubbLeNet: F}oveated imaging for visual discovery.
\newblock In {\em ICCV}, 2015.

\bibitem{Noroozi-16}
M.~Noroozi and P.~Favaro.
\newblock Unsupervised learning of visual representations by solving jigsaw
  puzzles.
\newblock In {\em ECCV}, 2016.

\bibitem{Pickup-14}
L.~C. Pickup, Z.~Pan, D.~Wei, Y.-C. Shih, C.~Zhang, A.~Zisserman,
  B.~Sch{\"o}lkopf, and W.~T. Freeman.
\newblock Seeing the arrow of time.
\newblock In {\em CVPR}, 2014.

\bibitem{Singh-12}
S.~Singh, A.~Gupta, and A.~A. Efros.
\newblock Unsupervised discovery of mid-level discriminative patches.
\newblock In {\em ECCV}, 2012.

\bibitem{thomee2016flickr}
B.~Thomee, D.~A. Shamma, G.~Friedland, B.~Elizalde, K.~Ni, D.~Poland, D.~Borth,
  and L.-J. Li.
\newblock Yfcc100m: The new data in multimedia research.
\newblock {\em CACM}, 59(2), Jan. 2016.

\bibitem{Vailaya-02}
A.~Vailaya, H.~Zhang, C.~Yang, F.-I. Liu, and A.~K. Jain.
\newblock Automatic image orientation detection.
\newblock {\em Trans. Image Processing}, 11(7):746--55, 2002.

\bibitem{Wei-18}
D.~Wei, J.~Y.~S. Lim, A.~Zisserman, and W.~T. Freeman.
\newblock Learning and using the arrow of time.
\newblock In {\em CVPR}, 2018.

\bibitem{EAST}
H.~W. Y. W. S. Z. W.~H. X.~Zhou, C.~Yao and J.~Liang.
\newblock East: An efficient and accurate scene text detector.
\newblock In {\em CVPR}, 2017.

\bibitem{zhou2016learning}
B.~Zhou, A.~Khosla, A.~Lapedriza, A.~Oliva, and A.~Torralba.
\newblock Learning deep features for discriminative localization.
\newblock In {\em Proceedings of the IEEE conference on computer vision and
  pattern recognition}, pages 2921--2929, 2016.

\end{thebibliography}


\begin{thebibliography}{1}\itemsep=-1pt

\bibitem{Doersch-15}
C.~Doersch, A.~Gupta, and A.~A. Efros.
\newblock Unsupervised visual representation learning by context prediction.
\newblock {\em ICCV}, 2015.

\bibitem{hql}
H.~S. Malvar, L.~wei He, and R.~Cutler.
\newblock High-quality linear interpolation for demosaicing of bayer-patterned
  color images.
\newblock {\em IEEE IntL Conf. on Acoustics,Speech,and Signal Procesing}, March
  2004.

\end{thebibliography}
}

\end{document}


\title{\articletitle}
\author{Zhiqiu Lin\textsuperscript{1} 
\qquad Jin Sun\textsuperscript{1,2} 
\qquad Abe Davis\textsuperscript{1,2} 
\qquad Noah Snavely\textsuperscript{1,2}\\
Cornell University\textsuperscript{1} \qquad Cornell Tech\textsuperscript{2}\\
}

\maketitle

\begin{abstract}
\noindent In this document we examine the effect that an image operation can have on the symmetries of an image distribution. We show how an operation's commutativity with a transformation can be used to predict how it will affect symmetries of a distribution with respect to that transformation. We then use this observation to analyze how operations such as Bayer demosaicing, JPEG compression, and random cropping, affect visual chirality, and show that our analysis accurately predicts the performance of deep networks trained on processed data.

\end{abstract}

\section{Introduction}
\label{sec:theory}
A key goal of our work on visual chirality is to understand how reflection changes what we can learn from image data. We can think of this change as the difference between two distributions: one representing the data, and the other representing its reflection. In our main paper, we examine this difference by training a network to distinguish between samples drawn from each distribution.
However, most image data undergoes extensive processing before it even leaves a camera, and it is easy to imagine how such processing could introduce asymmetry that makes it trivial to distinguish between original images and their reflections. For example, if a camera were to watermark every image with an asymmetric pattern, then nearly any distribution of images it produced would be chiral, even if the content of every image (watermark aside) were perfectly symmetric. 
This leads to an important question: %
when can we attribute visual chirality to the visual world, and when might it instead be a consequence of how we process images? To help answer this, we develop a theory relating the preservation of symmetry in an image distribution to the commutativity of image processing operations with symmetry transformations.

We begin by reviewing different types of symmetry and how they relate to data augmentation and machine learning (Section \ref{sec:definition}). Next, in Section \ref{sec:symmetrypreservation}, we define what it means for an operation to preserve symmetry, and derive various relationships between the commutativity of such operations with a transformation, and whether they preserve symmetries with respect to that transformation. 
Then, based on the theory we developed in Section \ref{sec:symmetrypreservation}, we introduce a simple technique that uses a small number of representative samples from a distribution to quickly estimate whether an imaging operation may introduce visual chirality into that distribution (Section \ref{sec:comres}). In Section \ref{sec:operations} we apply this technique to common digital image processing operations, including Bayer demosaicing and JPEG compression, to analyze the effect that they have on visual chirality and learning. Finally, in Section \ref{sec:glidesymmetry}, we extends our analysis to consider random cropping and show how it can sometimes be used to make chiral operations achiral.

We say that an operation is chiral if it can map an achiral (symmetric) distribution of images to a chiral (asymmetric) one, and say that it is achiral if it preserves symmetry. Some concrete results of our analysis include showing that Bayer demosaicing and JPEG compression are each achiral for certain image sizes and chiral for others, and that when combined they are chiral for all image sizes.
When either is combined with random cropping individually it becomes achiral. Finally, when demosaicing and JPEG compression are both applied in combination with random cropping, the resulting operation remains chiral.

Our theoretical and empirical results altogether suggest that nearly imperceptible chiral traces may be left in photos by non-commutative imaging pipelines, which has implications on self-supervised learning, image forensics, data augmentation, etc.

\section{Symmetry}\label{sec:definition}

We begin by reviewing what it means for a distribution to be symmetric, and for that symmetry to be preserved under different transformations. From this we will derive a relationship between the commutativity of an operation with a transformation, and the preservation of symmetries under that transformation.

\subsection{Terms \& Definitions:} 
We first define the terms of our analysis abstractly and with minimal assumptions to keep our conclusions as general as possible. 
We assume the following are given:
\begin{table*}[]
\begin{center}
\resizebox{\textwidth}{!}{%
\begin{tabular}{m{4.5cm} m{5cm} m{8cm}}
\toprule
\textbf{Term}&\textbf{Definition}&\textbf{Meaning in Learning Applications}\\ \midrule
A distribution \tD{} %
& $\D:\mathbb{R}^n\mapsto\mathbb{R}$ &
The underlying distribution our training data is drawn from for some task.\\
\midrule
A symmetry transformation \tT{}& $\T:\mathbb{R}^n\mapsto\mathbb{R}^n$, is associative and invertible&
E.g., horizontal reflection, or any other associative and invertible transformation to be used for data augmentation. \\
\midrule
A processing transformation \tJ{}&
$\J:\mathbb{R}^n\mapsto\mathbb{R}^n$, 
does \emph{not} have to be invertible&
Some combination of image processing operations, e.g., demosaicing and/or JPEG compression.\\
\midrule
A transformed distribution \tDj{}& %
$\Dj(\imagey) = \sum_{\imagex \in \J^{-1}(\imagey)}\D(\imagex)$
&
The distribution of training data after every element has been transformed by \tJ{}.\\
\bottomrule
\end{tabular}%
}
\end{center}
\caption{Terms and definitions used in derivations.}
\label{tab:symbols}
\end{table*}

\begin{itemize}
    \item A distribution $\D:\mathbb{R}^n\mapsto\mathbb{R}$ over some elements.
    \item A \emph{symmetry transformation} $\T:\mathbb{R}^n\mapsto\mathbb{R}^n$, which we assume to be invertible and associative. 
    \item A second (processing) transformation, $\J:\mathbb{R}^n\mapsto\mathbb{R}^n$, being applied to the domain of \tD{}.
    \item A \emph{transformed distribution} $\Dj:\mathbb{R}^n\mapsto\mathbb{R}$ obtained by applying \tJ{} to the elements of \tD{}:
    \beq
    \Dj(\imagey) = \sum_{\imagex \in \J^{-1}(\imagey)}\D(\imagex).
    \label{eq:dj}
    \eeq
\end{itemize}

\noindent{}These definitions intentionally omit assumptions that hold only in the specific case of analyzing visual chirality; for example, we do not assume that \tT{} is its own inverse, even though this holds for horizontal reflection. However, it is useful to remember how these abstract definitions apply to the concrete case of visual chirality, where \tD{} is a probability distribution over images, \tT{} is horizontal reflection, and \tJ{} is some kind of image processing operation. \tDj{} then describes the distribution associated with drawing images from \tD{} and subsequently applying \tJ{} ( i.e., the distribution of our training data if we apply \tJ{} to every training image).
So if $\imagesetX$ is a dataset of raw images that collectively approximate the distribution \tD{}, and \tJ{} is JPEG compression, then \tDj{} is the distribution approximated by $\J{}(\imagesetX{})$, which we get by applying JPEG compression to every image in \timagesetX{}. The summation in Equation \ref{eq:dj} accounts for the possibility that \tJ{} is non-injective,
in which case $|\J^{-1}(\imagey)|\ge{}1$ (i.e., 
\tJ{} maps multiple distinct inputs to the same output
). This is true, for example, of any lossy compression like JPEG.  In such cases, the probability of a transformed element \timagey{} is a sum over the probabilities associated with all inputs that map to \timagey.
For convenience, a summary of each term and its meaning in the context of visual chirality is also given in Table \ref{tab:symbols}.

\subsection{Symmetry of Elements \& Distributions}
It is important to distinguish what it means for an individual element to be symmetric, and what it means for that element to be symmetric under some distribution \tD{}.
We say that an element \timagex{} is symmetric with respect to a transformation \tT{} if: 
\beq
\imagex=\T\imagex
\label{eq:element_symmetry}
\eeq
while symmetry with respect to \tT{} under some distribution \tD{} is defined by the condition:
\beq
\D(\imagex)=\D(\T\imagex)
\label{eq:basic_symmetry}
\eeq

\noindent{}Which makes the distribution \tD{} itself symmetric if and only if Equation \ref{eq:basic_symmetry} holds for all \timagex{}.

Importantly, Equation \ref{eq:basic_symmetry} can hold even when Equation \ref{eq:element_symmetry} does not,
meaning that asymmetric elements may still be symmetric under the distribution \tD{}. In the context of computer vision, this happens when an image and its reflection are different (i.e., the image itself is not symmetric) but share the same probability under \tD{}.
On the other hand, as equivalence implies equivalence under a distribution, the symmetry of an individual element does imply symmetry under \tD{}, making the symmetry of an element a sufficient but not necessary condition for symmetry under a distribution. This makes any distribution over exclusively symmetric elements trivially symmetric; however, using \tT{} to augment data drawn from such a distribution would not be especially useful, as \tT{} would map every element to itself.

\section{Symmetry Preservation}\label{sec:symmetrypreservation}
In order to reason about asymmetry in visual content we need to understand how symmetry is affected by image processing. In particular, we need to know whether a processing transformation preserves symmetries in the original data. Without this knowledge, we cannot be certain whether asymmetries that we observe in images are properties of visual content, or of how that visual content was processed. Equations \ref{eq:element_symmetry} and \ref{eq:basic_symmetry} describe two distinct types of symmetry; the first relates two elements, while the second relates the images of these elements in \tD{}. If we consider the effect that an operation \tJ{} will have on each type of symmetry, we arrive at two different notions of what it means for symmetry to be preserved. The first describes whether the symmetry of individual elements is preserved.
It is defined by applying Equation \ref{eq:element_symmetry} to both \timagex{} and $\J{}(\imagex)$:
\beq
[\imagex=\T\imagex]\implies[\J(\imagex)=\T\J(\imagex)]
\label{eq:preservation1}
\eeq
The second type of symmetry preservation describes whether the symmetry of a distribution is preserved. It is defined by applying \tJ{} to Equation \ref{eq:basic_symmetry}:
\beq
[\D(\imagex)=\D(\T\imagex)]\implies[\Dj(\imagex)=\Dj(\T\imagex)]
\label{eq:distpreservation}
\eeq
Note that neither of Equations \ref{eq:preservation1} and \ref{eq:distpreservation} implies the other. For example, \tJ{} will trivially preserve element symmetry when applied to a domain that does not contain symmetric elements, but can easily break distribution symmetry. Likewise, we can break element symmetry while preserving distribution symmetry by permuting a uniform distribution of elements such that any symmetric element maps to a non-symmetric element.

\subsection{Commutativity \& Element Symmetry}
\label{sec:elementsym}
Our first type of symmetry preservation describes whether elements that are symmetric with respect to \tT{} remain so after applying \tJ{}. We now show that this holds if and only if \tJ{} commutes with \tT{} when applied to symmetric elements, meaning:
\beq
[\imagex=\T\imagex]\implies[\J(\T\imagex)=\T\J(\imagex)]
\label{eq:symcom}
\eeq
\begin{proposition}
\label{prop:symcom}
\tJ{} will preserve the symmetry of elements with respect to \tT{} (Equation \ref{eq:preservation1}) if and only if \tT{} and \tJ{} commute on symmetric elements (Equation \ref{eq:symcom}).
\begin{proof}
We start by showing that Equation \ref{eq:preservation1} implies Equation \ref{eq:symcom}. As equivalence implies equality under \tJ{}, we have:
\beq
[\imagex=\T\imagex]\implies[\J(\imagex)=\J(\T\imagex)]
\label{eq:preservation1b}
\eeq
And combining the right sides of Equations \ref{eq:preservation1} and \ref{eq:preservation1b} gives us Equation \ref{eq:symcom}. To show the other direction we start by applying \tJ{} to both sides of Equation \ref{eq:element_symmetry} to get the right side of Equation \ref{eq:preservation1b}. From here we use Equation \ref{eq:symcom} to commute \tT{} and \tJ{} and get Equation \ref{eq:preservation1}. This concludes the proof.
\end{proof}
\end{proposition}

Note this also means that, if $\J(\T\imagex)\neq\T\J(\imagex)$, then \tJ{} must break the symmetry of \timagex.

\subsection{Commutativity \& Element Mapping}
Notice that Proposition \ref{prop:symcom} falls short of establishing general commutativity of \tT{} and \tJ{}; it only applies to elements that are symmetric with respect to \tT{}. However, we can derive a stronger relationship related to general commutativity by considering whether \tJ{} preserves the \emph{mapping} of elements defined by \tT{}. To see this, note that \tT{} defines a map from each element $\imagex_a$ to another element $\imagex_b$, where:
\beq
\imagex_b=\T\imagex_a
\label{eq:mapping}
\eeq
We can think of Equation \ref{eq:mapping} as relaxing Equation \ref{eq:element_symmetry} to include asymmetric elements, for which $\imagex_a\neq\imagex_b$.  
We can then define the preservation of this mapping by \tJ{} as:
\beq
[\imagex_b=\T\imagex_a]\implies[\J(\imagex_b)=\T\J(\imagex_a)]
\label{eq:presmapping}
\eeq

From here, we can derive a stronger claim related to general commutativity.
\begin{proposition}
\label{prop:commutativitymap}
\tJ{} preserves the mapping established by \tT{} (Equation \ref{eq:presmapping}) if and only if \tJ{} commutes with \tT{}.
\end{proposition}
\begin{proof}
We start by showing that Equation~\ref{eq:presmapping} implies commutativity. As equivalence implies equality under \tJ{}, we have:
\beq
[\imagex_b=\T\imagex_a]\implies[\J(\imagex_b)=\J(\T\imagex_a)]
\label{eq:presmapb}
\eeq
And combining the right sides of Equations \ref{eq:presmapping} and \ref{eq:presmapb} gives us $\T{}\J(\imagex_a)=\J(\T\imagex_a)$. To show that commutativity implies Equation \ref{eq:presmapping}, we start by applying \tJ{} to both sides of Equation \ref{eq:mapping} to get the right side of Equation \ref{eq:presmapb}
\beq
[\imagex_b=\T\imagex_a]\implies[\J(\imagex_b)=\J(\T\imagex_a)]
\eeq
From here we commute \tT{} and \tJ{} on the right side to get Equation \ref{eq:presmapping}. This concludes the proof.
\end{proof}
From this we can also conclude that if \tJ{} does not commute with \tT{} then there must be some pair of elements $\imagex_a, \imagex_b$ such that Equation \ref{eq:presmapping} does not hold.

\subsection{The Symmetry of Distributions}
\label{sec:grouptheory}

We have shown that commutativity implies the preservation of element symmetry. 
Now we show that it also implies the preservation of distribution symmetry. This is a bit more complicated than the element case. Our approach, based on group theory, is to show
that when \tJ{} commutes with \tT{}, \tJ{} maps between disjoint cyclic groups generated by \tT{}.

\begin{proposition}
If $\J$ commutes with $\T$ and a distribution \tD is symmetric with respect to \tT, then the transformed distribution \tDj{} will also be symmetric with respect to \tT.
\label{prop:commutativity}
\end{proposition}
\begin{proof}
We first show that \tJ{} defines a mapping between disjoint cyclic subgroups. We then show that this map is a homomorphism, which we use to relate $\Dj(\imagex)$ to $\Dj(\T(\imagex))$.

Since \tT{} is associative and invertible, we can use it to partition our domain into disjoint cyclic subgroups \tgset{\imagexi} generated by \tT{}:
\beq
\gset{\imagexi}=\{...,\T^{-1}\imagexi,\imagexi,\T\imagexi,\T^2\imagexi,\T^3\imagexi,...\}
\eeq
where the identity element $\imagexi$ of each group can be chosen as any arbitrary element within the group. 
We refer to the set of such group \textit{generators} as $\mathcal{G}_{\T}$. 
The group operation $\gropr$ can be thought of as a permutation of each specific cyclic group relative to its identity element:
\beq
\T^a\imagexi\gropr\T^b\imagexi = \T^{a+b}\imagexi
\eeq
As each such subgroup shares the same group operation and is closed under that operation, any two $\gset{\imagexi}$ must either be equivalent or disjoint. The order $|\gset{\imagexi}|$ of each subgroup depends on the symmetries of $\imagexi$ with respect to \tT{}. For example, if \tT{} is simple reflection about a particular axis then $|\gset{\imagex_i}|=1$ for images $\imagex_i$ that are symmetric about that axis, and $|\gset{\imagex_i}|=2$ for images that are asymmetric about that axis.

Now consider how \tJ{} transforms each of the subgroups $\gset{\imagex_i}$:
\beq
\J\gseti=\{...\J\T^{-1}\imagexi,\J\imagexi,\J\T\imagexi,\J\T^2\imagexi,...\}
\label{eq:jsymset1}
\eeq
If \tJ{} commutes with \tT{}, we can rewrite the above as
\beq
\J\gseti{}=\{...\T^{-1}\J\imagexi,\J\imagexi,\T\J\imagexi,\T^2\J\imagexi,...\}
\eeq
giving us
\beq
\J\gseti{}=\gset{\J\imagexi}
\label{eq:gtog}
\eeq
This shows that \tJ{} maps cyclic subgroups generated by \tT{} to cyclic subgroups that can be generated by \tT{}.

Symmetry with respect to \tT{} can be restated as the condition that all elements within common cyclic subgroups generated by \tT{} share the same probability. 
In other words, for each cyclic subgroup $\gseti{}$, all elements of the subgroup have the same probability under $\D$, i.e., all elements have probability $\D(\imagexi)$.
It is therefore sufficient for us to show that the map $\J:\gseti\mapsto\gset{\J\imagexi}$ is a homomorphism, as the first isomorphism theorem ensures the same number of equal-probability elements from $\gseti{}$ will map to each element of $\gset{\J\imagexi}$.

Recall that a homomorphism $h:G\mapsto H$ is defined by the relation $h(u\gropr v)=h(u)\gropr h(v)$. It is simple to show that this holds for \tJ{} and our cyclic subgroups when \tJ{} commutes with \tT{}:
\beq
\begin{split}
  \J(\T^{a}\imagexi\cdot\T^{b}\imagexi)
 & = \J(\T^{a+b}\imagexi)\\
 & = \T^{a+b}\J(\imagexi)\\
 & = \T^a\J(\imagexi)\cdot\T^b\J(\imagexi)\\
 & = \J(\T^a\imagexi)\cdot\J(\T^b\imagexi)\\
\end{split}
\eeq
This is sufficient to prove our proposition.
For completeness, we also reformulate \tDj{} in terms of the cyclic subgroups $\gset{\J\imagex_i}$. We will use the notation 
$\mathbf{1}_{\gset{\imagex_i}}$ 
to denote an indicator distribution that maps every element of $\gset{\imagex_i}$ to 1, and every other element to 0. Note that any distribution we can represent as the weighted sum of $\mathbf{1}_{\gset{\imagex_i}}$ must preserve symmetry with respect to \tT{}. We can express \tD as:
\beq
\D=\sum_{\imagex_i\in \generators}\D(\imagexi)\cdot \mathbf{1}_{\gset{\imagex_i}}
\label{eq:dsumofcg}
\eeq
Now, using the first isomorphism theorem to account for the case where \tJ{} is non-injective, we can combine Equation \ref{eq:dj} and \ref{eq:dsumofcg} to write \tDj{} as
\beq
\Dj = \sum_{i}(\D(\imagexi)|\mathrm{ker}\ {\J\imagexi}|)\cdot \mathbf{1}_{\gset{\J\imagexi}}
\label{eq:djfirsthomomorphism}
\eeq
where $\mathrm{ker}\ {\J\imagexi}$ is the kernel of $\J:\gseti\mapsto\gset{\J\imagexi}$.\footnote{Recall that the \textit{kernel} of a homomorphism $\J$, $\mathrm{ker}\ \J$, is the subset of elements that $\J$ maps to the identity element.} This concludes our proof.
\end{proof}

\subsection{Permuted Commutativity}
\label{sec:tgroups}
From Proposition \ref{prop:symcom} we can conclude that if \tJ{} does \emph{not} commute with \tT{} when applied to some symmetric element \timagex{} then $\J{}(\imagex)$ will not be symmetric with respect to \tT{} (this is also simple to prove independently). However, our proof of Proposition \ref{prop:commutativity} is not bi-directional; we only show that commutativity implies distribution symmetry will be preserved. What, then, can we conclude about operations that do not commute with \tT{}? 

The first thing to note is that non-commutativity does \emph{not} imply that distribution symmetry will be broken. There are various ways for symmetry to be preserved even when \tJ{} and \tT{} do not commute, but here we consider a case where groups of operations that do not commute with \tT{} individually combine to preserve distribution symmetry. For instance, imagine that given a training set of images drawn from a distribution $\D$, and that we generate a new training set by applying multiple random crops to each original image---we can think of each different crop offset $j$ as a different transformation $\J_i$ applied to the original distribution, and the accumulation of all random crops to reflect a new, accumulated distribution. As such, this accumulation of transformed images is particularly relevant to computer vision, as it will help us explain an effect that random cropping can have on bias introduced by data augmentation.

We proved Proposition \ref{prop:commutativity} by showing that \tJ{} formed a homomorphism between cyclic subgroups $\gseti{}$ and $\gset{\J\imagexi}$. We now consider the case where \tDj{} is the sum of multiple such homomorphisms $\Jj{j}$, as would result from accumulating the results of multiple transformations $\Jj{j}$:
\beq
\Dj(\imagex)=\sum_{j}\sum_{\imagex_{i}:\Jj{j}\imagex_i=\imagex}\D(\imagex_i)
\label{eq:djgdef}
\eeq
In this case, the sum of symmetric distributions is a symmetric distribution, which tells us that \tDj{} will still be symmetric. Now note that by permuting the elements on the right side of Equation \ref{eq:djgdef}, we can define a new set of transformations that sum to the same \tDj{}, ensuring that symmetry remains preserved.

\begin{figure}[th]
\begin{centering}
\includegraphics[width=0.45\textwidth]{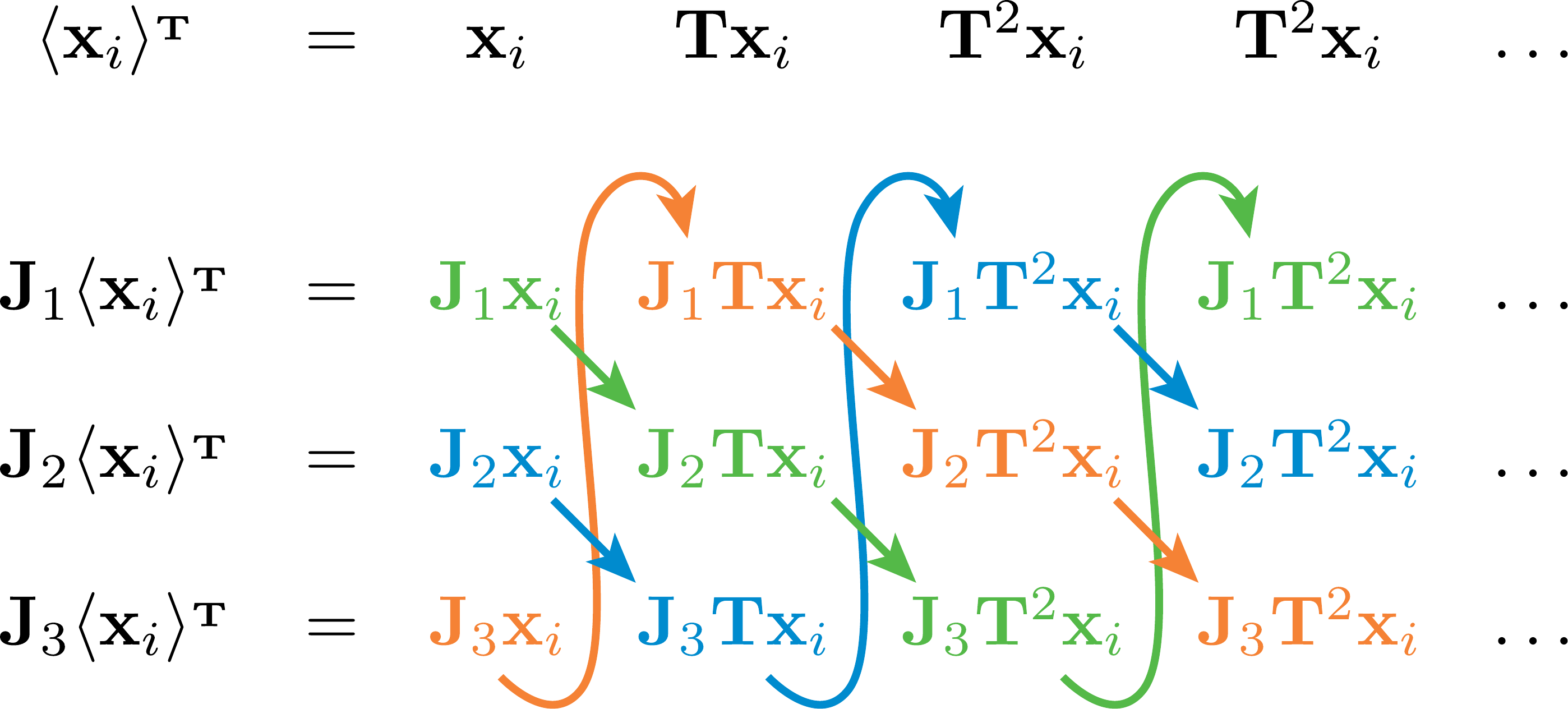}
\caption{\textbf{Permuted Commutativity.} We can permute the elements on the right side of Equation \ref{eq:djgdef} to define a set of processing transformations $\Jj{j}$ that may not commute with \tT, but in aggregate still sum to a symmetric distribution. These $\Jj{j}$ are characterized by a permuted commutativity relationship. In this example, that relationship is given by $\T\Jj{a}=\Jj{b}\T$, where $a=k\pmod{3}$ and $b=k+1\pmod{3}$.}\label{fig:glidecom}
\end{centering}
\end{figure}
This means that if we can find a permutation of $\Jj{j}$ such that the aggregated transformed distribution still maintain symmetric when $\T\Jj{a}=\Jj{b}\T$ for $a$ and $b$ are indices of a found permutation (Figure \ref{fig:glidecom}). This notion of permuted commutativity will lead to a definition of \textit{glide commutativity} in section \ref{sec:glidesymmetry} concerning groups related by translation in 2D image plane.

\section{Commutative Residuals}\label{sec:comres}
Propositions~\ref{prop:commutativitymap} and \ref{prop:commutativity} establish a connection between the commutativity of two operations (a processing operation and a symmetry transformation), and the preservation of two kinds of symmetries. Notably, commutativity guarantees that the symmetry of a distribution under a transformation is preserved. How do we apply this finding in practice on a specific distribution of images and a specific processing operation?

It can be difficult to model complex processing operations like JPEG compression and Bayer demosaicing analytically, and it may be the case that such operations commute with a transformation when applied to certain inputs, and not when applied to others. Furthermore, the impact of non-commutativity is not binary: if we think of the asymmetries introduced by an operation as some signal indicating, for example, whether an image has been flipped, then it is useful to consider the magnitude of that signal relative to variations in the distribution that contains it. These concerns lead us to derive a numerical measure of commutativity that we can evaluate on representative samples of a distribution to gauge the strength of assymmetries introduced by an operation. 
We define $\comresi{}(\imagex)$, the \emph{commutative residual image} of operation \tJ{} with respect to transformation \tT{} on the image \timagex{}, as follows:
\beq
\comresi(\imagex)= \J(\T(\imagex))-\T(\J{}(\imagex))
\label{eq:commutativeresidual}
\eeq
We can get a rough measure of the commutativity between an imaging processing step and a transformation on some representative samples \timagex{} by looking at the value of $|\comresi(\imagex)|$, which we summarize by its average across all pixels, $\comres{}(\imagex)$. We refer to $\comres{}(\imagex)$ as a \textit{commutative residual}. A commutative residual of 0 on a particular image \timagex{} means that $\T$ and $\J$ commute for that image, and a non-zero commutative residuals means that they do not commute for that image. As the derivations in Section~\ref{sec:symmetrypreservation} show, if the commutative residual is 0 for all elements of a distribution (i.e., the processing operation commutes with $\T$), the symmetry will be preserved. If not, symmetries may be broken.

\begin{figure*}[ht]
\begin{centering}
\includegraphics[width=0.9\textwidth]{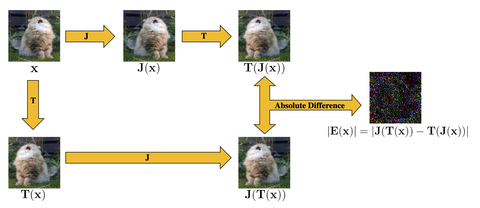}
\caption{\textbf{Example commutative residual image:} This figure illustrates the application of the commutative residual method to a natural image. Here \tT{} is the horizontal reflection operation, and \tJ{} is the composition of Bayer demosaicing and JPEG compression. The image used above has a width of 100px. For better visualization of the imperceptible differences shown in the residual image, we scale the resulting residual by a factor of 10. Consistent with the results in Figure~\ref{fig:commutativity_vs_width}, the residual image is not zero (which would be all black), i.e., the commutative residual is non-zero. }\label{fig:commutativity_residual_example}
\end{centering}
\end{figure*}

\paragraph{Commutative Residuals for Mirror Reflections.}
An alternative intuition of commutative residuals can be arrived at in the case where \tT{} is its own inverse, as is true of mirror reflections. Consider the effect of $\J{}$ on a distribution represented by a dataset with two elements, $\D=\{\imagex, \T(\imagex)\}$. 
This simple distribution is trivially symmetric, since $\D$ is closed under $\T$.

But what happens when we apply \tJ{}? \tD{} becomes $\Dj{}=\{\J(\imagex), \J(\T(\imagex))\}$, and we can measure the asymmetry of this new distribution by taking the difference between one element and the reflection of the other:
\beq
\J(\T(\imagex))-\T(\J(\imagex))
\eeq
which is precisely how we define the commutative residual image above.
Figure~\ref{fig:commutativity_residual_example} shows an example computation of a commutative residual image when $\T$ is image flipping and $\J$ is the composition of Bayer demosaicing and JPEG compression.

\subsection{Evaluating the Chirality of Operations}
We propose two methods to evaluate the chirality introduced to an originally achiral distribution \tD{} by an operation \tJ{}. The first approach, based on the theory we have derived about commutativity, is to evaluate the commutative residual with respect to \tJ{} on a small representative set of sample images. The second method, as described in the main paper in the context of analyzing real image datasets, is to train a neural network to empirically distinguish between flipped and unflipped images sampled from a much larger, symmetric dataset after transforming every image in that dataset by \tJ{}. 
Since we are interested in demonstrating the possibility of introducing chirality through low-level imaging operators, we study image distributions that are originally symmetric to ensure that any learned chirality cues can be attributed solely to the effect of \tJ{}.

\begin{figure}[t]
\begin{centering}
\includegraphics[width=0.3\textwidth]{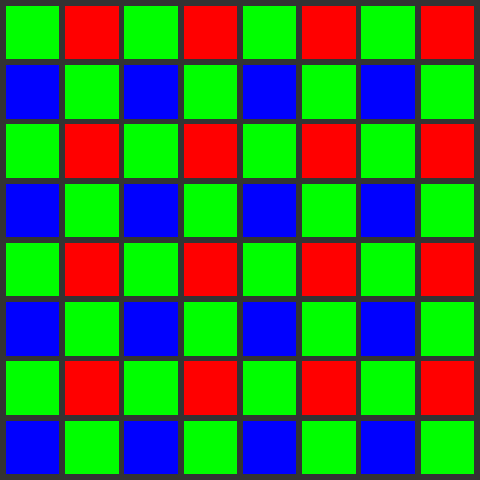}
\caption{\textbf{Example 8$\times$8 Bayer pattern mosaic}: A typical Bayer filter mosaic consists of tiled 2$\times$2 blocks of pixels with two green filters and one red and one blue filter. Note that a even-sized Bayer filter, like the one pictured, is asymmetric (mirror flipped version is not equal to itself), while an odd-sized version of this filter pattern would be symmetric.}\label{fig:bayer}
\end{centering}
\end{figure}

\begin{figure*}[ht]
\begin{centering}
\includegraphics[width=0.9\textwidth]{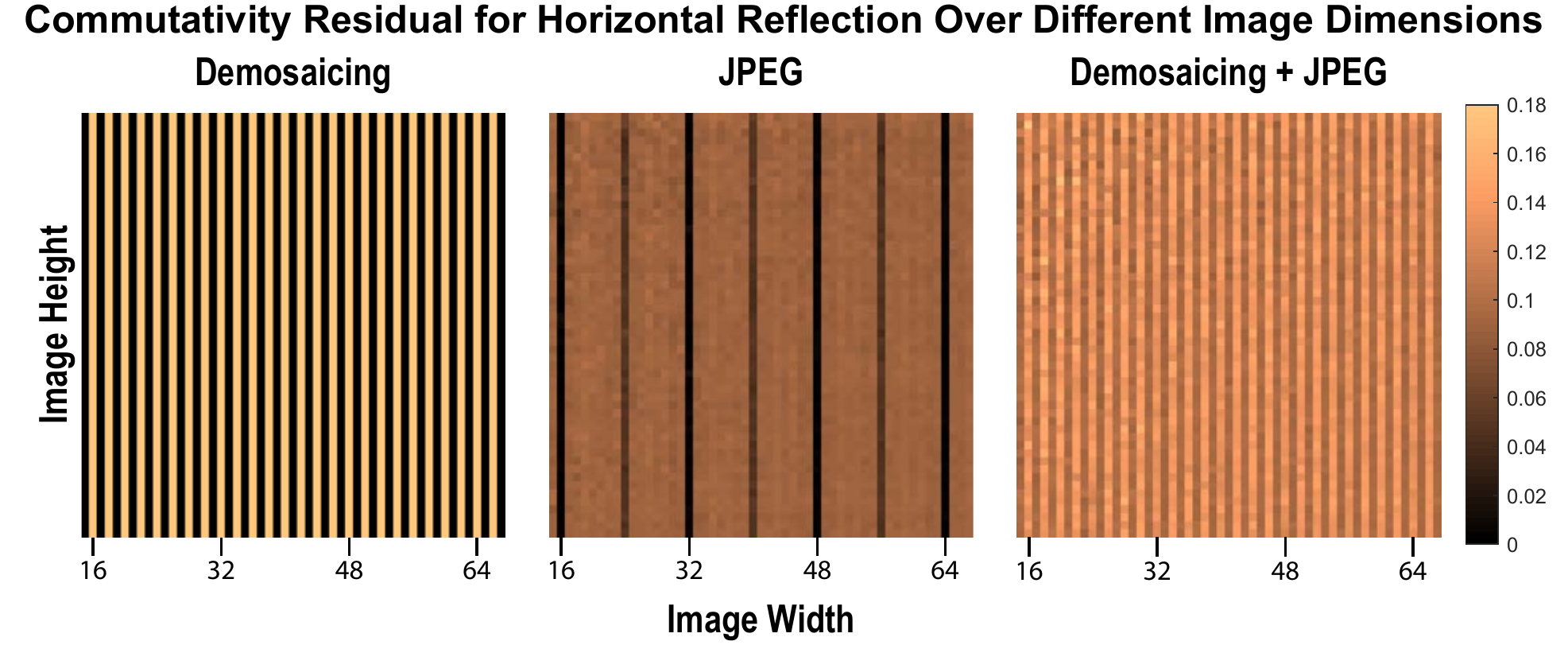}
\caption{\textbf{Commutativity residuals for demosaicing (left), JPEG compression (middle) and their composition (right)}: Each image shows how commutativity residual, measured in absolute average percent error per pixel, varies with different image sizes. For integers $n$ we see commutativity in demosaicing at image widths of $2n-1$ (i.e., odd widths), and in JPEG compression at widths of $16n$. We do not see commutativity when both are applied.}\label{fig:commutativity_vs_width}
\end{centering}
\end{figure*}

\section{Analysis of Demosaicing \& JPEG Compression 
}
\label{sec:operations}
With the theoretical tools derived in previous sections, we evaluate two standard imaging processes: Bayer demosaicing and JPEG compression. We analyze when these two operations (and their composition) will preserve existing symmetries in a distribution of images, and when they may break them. 
In real camera systems, Bayer demosaicing and JPEG compression are typically two operations in a much larger image signal processing pipeline. We analyze these two operations specifically because (a) they are ubiquitous and implemented in most cameras, and (b) they have interesting symmetry properties, as we will show below.
We begin with a brief summary of these two operations.

\medskip
\noindent \textbf{Bayer filters and demosaicing.}
Many modern digital cameras (including cellphone cameras) capture color by means of a square grid of colored filters that lies atop of the grid of photosensors in the camera. An 8$\times$8 example of such a color filter grid, known as a Bayer filter mosaic, is shown in Figure~\ref{fig:bayer}. In such cameras, each pixel's sensor measures intensity for a single color channel (red, green, or blue), and so to produce a full color image at full resolution, we must interpolate each color channel such that each pixel ultimately has an R, G, and B value. This interpolation process is known as \emph{demosaicing}. 
For our analysis we assume, as is typical, that a Bayer filter mosaic pattern consists of a tiled 2$\times$2 element (GRBG in the case of Figure~\ref{fig:bayer}) and we consider the demosaicing method of Malvar~\cite{hql}.

The 8$\times$8 Bayer filter mosaic in Figure~\ref{fig:bayer} has interesting symmetry properties. The 8$\times$8 pattern as a whole is asymmetric---flipping it horizontally will result in a red pixel in the upper-left corner, rather than a green pixel. The same is true for any even-sized Bayer filter mosaic. However, from the perspective of the center of any pixel, the pattern is locally symmetric. Moreover, if we imagine a 9$\times$9 version of this mosaic (or indeed any odd-sized pattern), that mosaic would be symmetric.

\medskip
\noindent \textbf{JPEG compression.} 
JPEG is one of the most common (lossy) image compression schemes. There are two main ways that JPEG compresses image data. First, it converts images into the $\mathrm{Y'C_{b}C_{r}}$ colorspace and downsamples the chroma channels ($\mathrm{C_b}$ and $\mathrm{C_r}$), typically by a factor of two. Then it splits each channel into a grid of $8\times8$ pixel blocks and computes the discrete cosine transform (DCT) of each block. In the luminance ($\mathrm{Y'}$) channel, each block covers an 8$\times$8 pixel region of the original image, while for the chroma channels, each block corresponds to a 16$\times$16 pixel region in the original image, due to the 2$\times$ downsampling. Finally, the DCT of each block is strategically quantized to further compress the data at low perceptual cost.

For the purposes of our analysis, one noteworthy aspect of JPEG compression is that for images with dimensions that are not a multiple of 16, there will be boundary blocks that do not have a full 8$\times$8 complement of pixels. These are handled specially by the JPEG algorithm, which can lead to breaking of symmetry for such images because the special boundary blocks are always at the right (and bottom) edges of the image, never at the left (and top) edges.

\subsection{Commutative Residuals and Image Size}\label{sec:nocropsize}
As an initial experiment, we generate a completely random image with random dimensions (i.e., choosing the width and height uniformly at random from some uniform distribution, and then selecting each value for each color channel at random from the range $[0, 255]$).
Then we compute commutative residuals under the operations of (1) Bayer demosaicing (i.e., first synthetically generating a Bayer mosaic, then demosaicing it), (2) JPEG compression, and (3) the composition of these operations.

If we actually perform this experiment for randomly sized images, then under demosaicing, commutative residuals are nonzero about half of the time, and under JPEG compression, they are  nonzero over 90\% of the time. But if we sample over different image sizes more systematically, a pattern begins to emerge. 

Figure~\ref{fig:commutativity_vs_width} visualizes commutative residuals for random noise images as a function of image width and height for the three operations described above. We can see that demosaicing appears to commute with image flipping (and therefore preserve symmetries) for images with odd widths, while JPEG compression appears to preserve symmetries for image with widths that are divisible by 16. 
Finally and most notably, commutativity never seems to hold for the composition of demosaicing and JPEG compression for any width. We can explain this result by considering the geometry of Bayer patterns and JPEG block grids. Bayer patterns (Figure~\ref{fig:bayer}) have horizontal symmetry when reflected about any line centered on a pixel column, while the JPEG block grid, which consists of 8$\times$8 blocks that correspond to 8$\times$8 or 16$\times$16 blocks of the original image, is horizontally symmetric only around grid lines, which rest between columns at 16-pixel intervals.
A corollary is that the combination of demosaicing followed by JPEG compression can never be commutative with respect to flipping because these two imaging processes never have zero commutative residual for the same image width (since multiples of 16 are never odd).

The black-box analysis \footnote{All analysis and experiments in this and the next section are available in Python at \url{https://github.com/linzhiqiu/digital_chirality}.} of commutative residuals shown in Figure \ref{fig:commutativity_vs_width} reveals the grid structures underlying these processing algorithms, and illustrates how each grid structure impacts preservation of symmetries. When the commutative residual for any transformation for a given image width is zero, we know that this transformation preserves the symmetry of the original distribution with such width. 
\textit{\textbf{Hence, a key result is that demosaicing followed by JPEG compression always yields asymmetric distributions for images of arbitrary widths and heights even when the input distribution of these images is symmetric.}}
Since the combination of these two operations is very standard in imaging pipelines, we can expect results on synthetic data to apply to real images as well. 

When the commutative residual is non-zero, we hypothesize that in practice symmetries will be broken, i.e., a non-commutative imaging process will make an originally achiral distribution chiral. To test this hypothesis, we trained deep neural networks on three synthetic achiral distributions of Gaussian noise images, corresponding to three different square images sizes: one with odd width ($99\times 99$), one with even width \textbf{not} divisible by 8 ($100\times 100$), and one that is a multiple of 16 ($112\times 112$). 
To generate a sample image from each distribution, for each pixel, we sample its color value from a per-channel Gaussian distribution. The mean of each color channel (in the range $[0,1]$) was set to $(0.6, 0.5, 0.9)$ (for red, green, and blue, respectively), and the standard deviation to $(0.3, 0.25, 0.4)$. We use per-channel means and standard deviations (rather than the same Gaussian distribution for all channels) to reduce the source of symmetries present other than symmetry with respect to \tT. An example image from this distribution, before and after each processing step, is shown in Figure~\ref{fig:samples}.

\begin{figure}[t]
\centering
\begin{tabular}{cc}

\includegraphics[width=0.45\columnwidth]{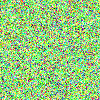} &
\includegraphics[width=0.45\columnwidth]{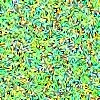} \\
(a) Original Image & (b) Demosaicing \\
\includegraphics[width=0.45\columnwidth]{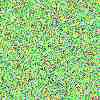} &
\includegraphics[width=0.45\columnwidth]{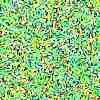} \\
(c) JPEG compression & (d) Demosaicing + JPEG \\
\end{tabular}
\caption{\textbf{A sample image from our Gaussian noise image distribution after different imaging operations}. This image is of size (100,100) and is generated using the Gaussian noise method described in Section~\ref{sec:nocropsize}.}\label{fig:samples}.
\end{figure}

If we apply each processing operation to all images from these three distribution over three image sizes, our hypothesis predicts that the operations will either preserve or break achirality according to Table~\ref{tab:nocrop}.

We train a binary chirality prediction (flip/no-flip) network using the same ResNet model as in the main paper (with randomly initialized weights) for each of these nine datasets (3 image sizes times 3 processing operations), with learning rates obtained from log-scale grid searches. As predicted by our hypothesis, trained network models can never achieve more than $50\%$ test classification accuracy on processed distributions that our analysis suggests to be achiral (i.e., commutative residual is zero). And, intriguingly, our trained network models achieve near perfect classification accuracy on processed distributions resulting from non-commutative imaging processes. This experiment hence gives empirical evidence that non-commutativity of a processing operation strongly suggests a loss of achirality.

\begin{table}[t]
\begin{center}
\begin{tabular}{lccc}
\toprule 
{Imaging Operation} & \multicolumn{3}{c}{Image size} \\
\cmidrule{2-4}
& 99 & 100 & 112 \\
 \midrule
Demosaicing & A & C & C\\
JPEG & C & C & A\\
Demosaicing+JPEG & C & C & C\\  \bottomrule
\end{tabular}
\end{center}
\caption{\textbf{Predicted chirality of three (initially achiral) Gaussian noise image distributions (corresponding to three different square image sizes) under each of three processing schemes.} `C' means chiral, and `A' means achiral. Explanation: 99px images should remain achiral under demosaicing, since the images have odd size. 112px images should remain achiral under JPEG compression since they have size divisible by 16. Everything else becomes chiral as hypothesized. We verify this table empirically by training network models on the nine distributions resulting from these transformations.}
\label{tab:nocrop}
\end{table}

\begin{figure*}[t]
\centering
\begin{tabular}{cccc}
\rotatebox{90}{Bilinear Resizing}
\includegraphics[width=1.5in]{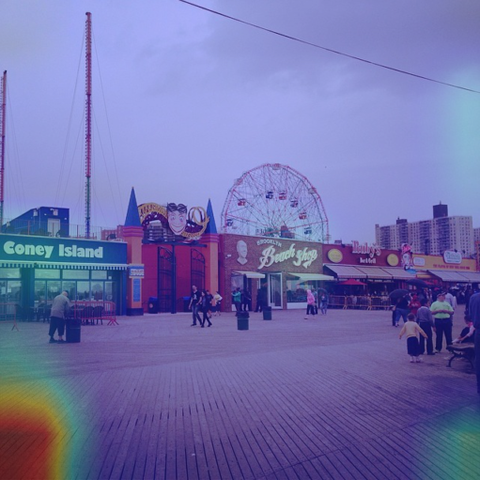} &
\includegraphics[width=1.5in]{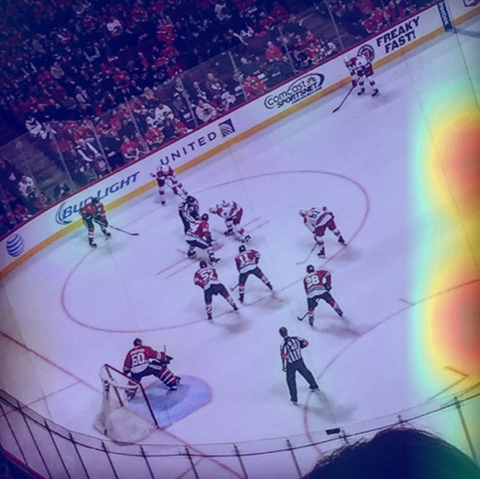}&
\includegraphics[width=1.5in]{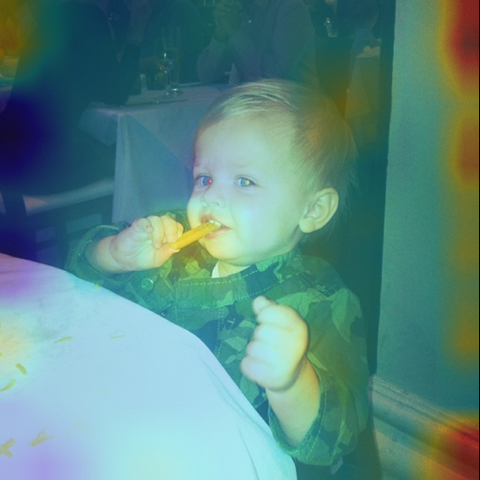}&
\includegraphics[width=1.5in]{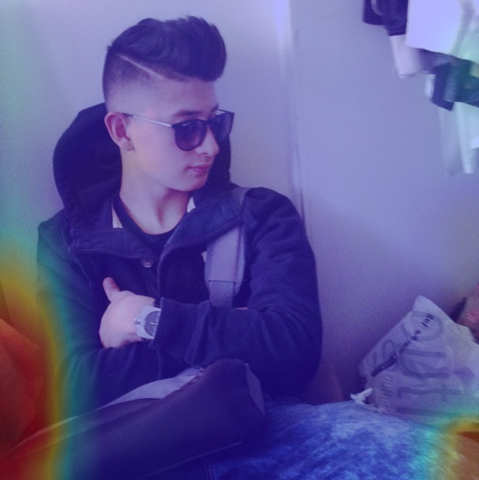}\\
\rotatebox{90}{Random Cropping}
\includegraphics[width=1.5in]{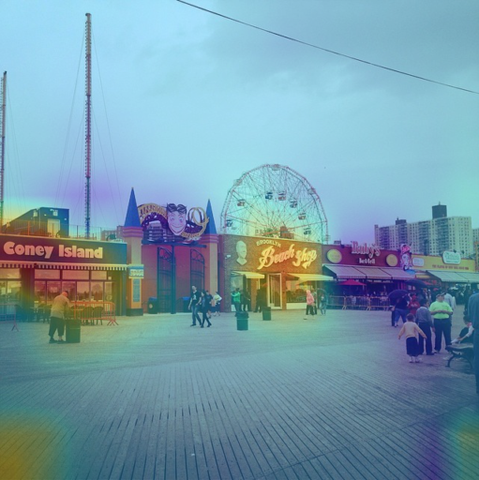}&
\includegraphics[width=1.5in]{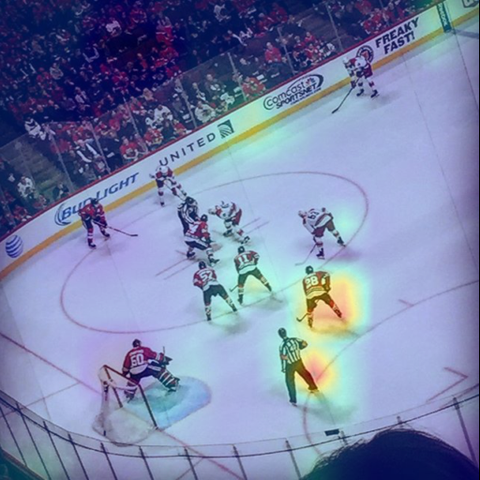}&
\includegraphics[width=1.5in]{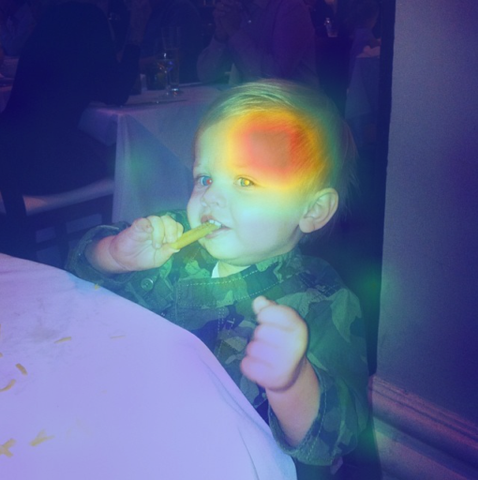}&
\includegraphics[width=1.5in]{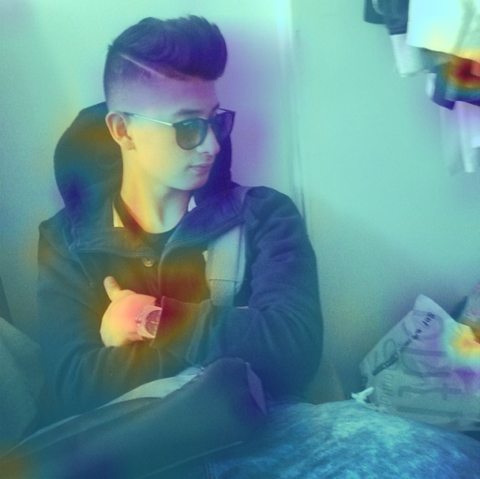}
\end{tabular}
\caption{Class Activation Maps (CAM) resulting from two preprocessing procedures used in training ImageNet-pretrained models on the chirality task: (top row) simple bilinear resizing and (bottom row) random cropping. Recall from the main paper that the CAM tends to fire on discriminative regions for classification. Note the heavy focus on edge and corner regions on bilinear resized images, likely due to edge artifacts caused by JPEG compression or demosaicing (or both). These artifacts disappear when random cropping is applied.}
\label{fig:resizing}
\end{figure*}

Note that this analysis assumes that we use the whole images after Bayer demosaicing and/or JPEG compression, i.e., no cropping. These results nicely mirror the situation of training networks on \emph{real} images with no random cropping, as described in the main paper. Figure~\ref{fig:resizing} shows that networks trained to classify chirality on resized (but not cropped) Instagram images often seem to focus on image evidence near boundaries (first row), which we hypothesis is due exactly to the kinds of chiral boundary artifacts discussed in this section in the context of JPEG compression. On the other hand, training with random cropping data augmentation yields networks that appear to focus on much more high-level features (second row). In the next section, we discuss the interaction of processing with random cropping (or image translation) and how the addition of random cropping can either make a chiral imaging process achiral, or can sometimes still introduce chirality.

\begin{figure}[t]
\begin{centering}
\includegraphics[width=0.9\columnwidth]{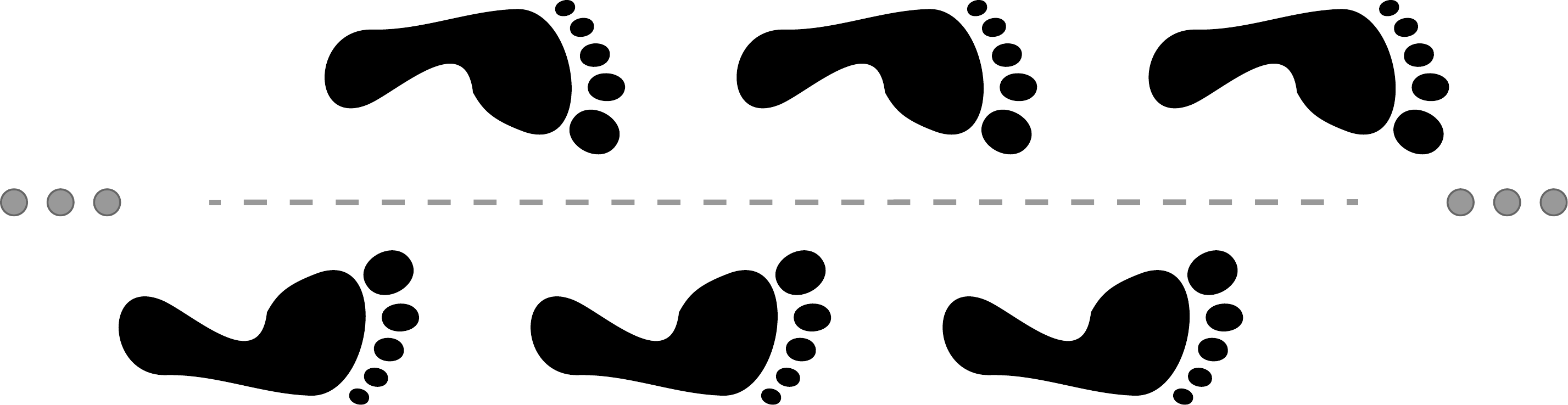}
\caption{\textbf{Glide Symmetry}: Human footprints often exhibit glide symmetry. The infinitely repeating footstep pattern shown here is equivalent to the reflection of a shifted version of itself.}\label{fig:glidesymmetry}
\end{centering}
\end{figure}

\section{Random Cropping and Glide Symmetry}
\label{sec:glidesymmetry}
Because our analysis makes few assumptions about $\T$, $\J$ and $\D$, we can apply it to other symmetries and data augmentation strategies used in computer vision. For example, translational invariance is a common and useful prior in images that is often applied to data through the use of random crops as a type of data augmentation. Here we consider how our theory can be used to understand the effect of random cropping on training.

\subsection{Random Cropping as a Symmetry Transform}
Doersch~\etal \cite{Doersch-15} found that when they trained a network to predict the relative position of different regions in an image, it would sometimes ``cheat” by utilizing chromatic aberration
for prediction. We can use our observation about commutativity to explain this behavior by considering a family of transformations in the 2D image plane. The self-supervision task used in Doersch~\etal requires the network to distinguish between different translations, which is only possible when the following symmetry does not hold:
\beq
\D(\imagex)=\D(\Tv(\imagex)),
\label{eq:translation}
\eeq
where \tTv{} is translation by some vector $\tvec\in\mathbb{R}^2$.
Our commutativity analysis tells us that this symmetry can be broken by any \tJ{} that does not commute (or glide commute) with translation. This agrees with the findings of Doersch~\etal that the network was able to ``cheat" using artifacts caused by chromatic aberration, which is not translation-invariant, as its effect is spatially varying.

\subsection{Random Cropping as an Image Operation}
If we revisit our analysis of commutative residuals under an assumption of translation invariance,
we can draw new conclusions about the chirality of demosaicing and JPEG compression. In particular, by incorporating translation invariance in the form of random cropping, we can change the chirality of these operations by creating the kind of permuted commutativity described in Section \ref{sec:tgroups}. In the case where permuted commutativity happens among groups related by translation, we call it \emph{glide-commutativity}.

\begin{figure*}
\begin{centering}
\includegraphics[width=0.9\textwidth]{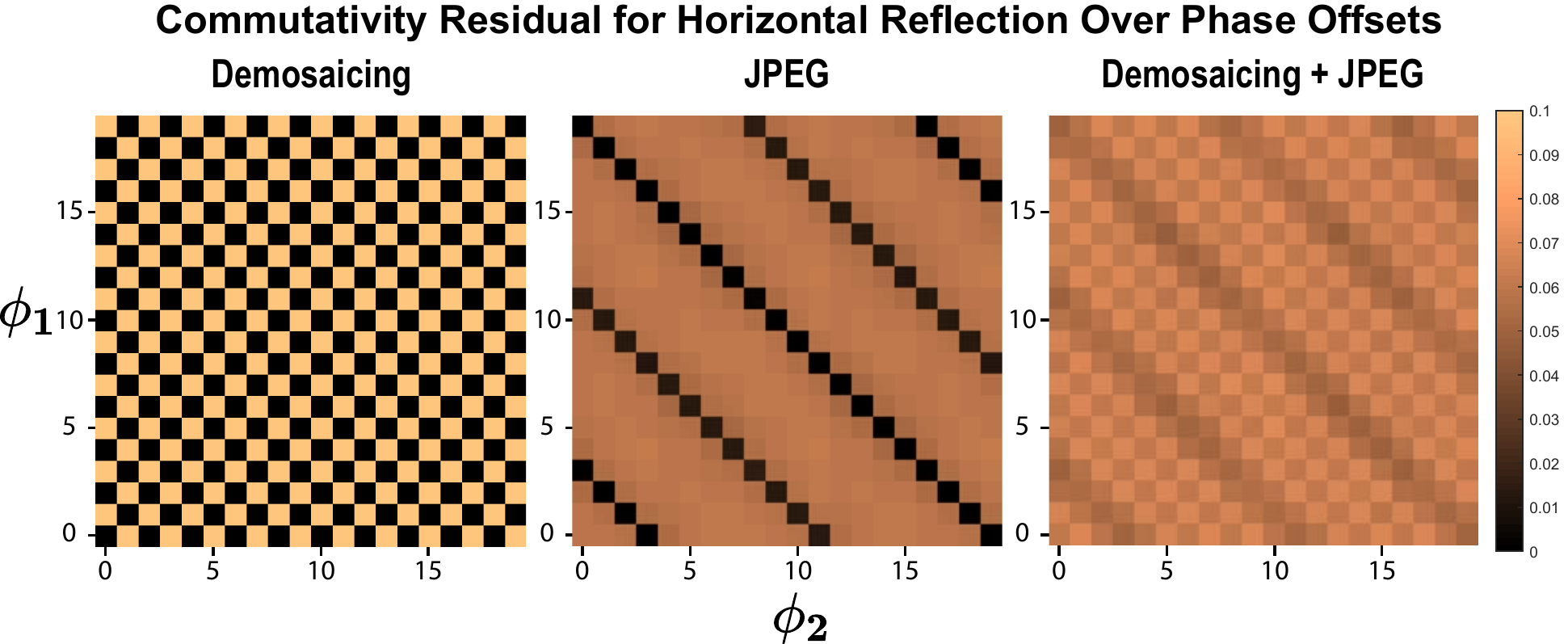}
\caption{\textbf{Glide Commutativity Residuals for demosaicing (left), JPEG compression (middle) and their composition (right)}: Each image shows the glide commutativity residual, measured in absolute average percent error per pixel, measured over different phase shifts.
For certain $\phi_1$ and $\phi_2$ we see commutativity in demosaicing and in jpeg compression alone. 
We do not see commutativity when both are applied.}\label{fig:glide_commitativity}
\end{centering}
\end{figure*}

To test for glide-commutativity, we must look for the permutation pattern described in Section \ref{sec:tgroups}. To do this, we first define a way of phase-shifting $\T(\J(\imagex))$ and $\J(\T(\imagex))$. For this, we define $\J\T_{\phi}(\imagex)$ and $\T\J_{\phi}(\imagex)$ as the process of:
\begin{packed_enum}
    \item Padding \timagex{} with a large, constant number of pixels on all sides.
    \item Translating the padded image by $\phi$.
    \item Applying \tT{} then \tJ{} for $\J\T_{\phi}(\imagex)$, or \tJ{} then \tT{} for $\T\J_{\phi}(\imagex)$.
    \item Translating by $\T(-\phi)$.
    \item Cropping out the previously padded pixels.
\end{packed_enum}
This has the effect of performing \tJ{} and \tT{} as if the image had occurred at a translation of $\phi$ from its original position. For grid-based algorithms like demosaicing and JPEG compression, this effectively phase-shifts the grid structure used in the algorithm.

To test for glide-commutativity we simply look for some repeating pattern of zeros in residuals of the form:
\beq
\comresi(\imagex,\phi_1,\phi_2)= \J\T_{\phi_1}(\imagex)-\T\J_{\phi_{2}}(\imagex)
\eeq
This pattern of zeros describes the permutation pattern described in Section \ref{sec:tgroups}. As the results in Figure \ref{fig:commutativity_vs_width} show, we verified that the vertical components of $\phi_{1}$ and $\phi_{2}$ do not matter. We therefore set them only to vary in the x dimension of the image. Figure \ref{fig:glide_commitativity} shows the residuals calculated for a range of phase shifts. We see that both demosaicing and JPEG compression appear to be glide-commutative due to the regular repeating pattern of zeros. However, the combination of demosaicing and JPEG compression does not appear to be glide-commutative, and we can see this is because zeros always occur at different phase shifts for each of the two operations.

\subsection{Empirical chirality in the presence of random crops}
The analysis from the previous section has simple implications (in terms of random cropping on images): (1) The distribution of random crops (while avoiding cropping from the boundary of 16 pixels) from an originally achiral distribution of images that has undergone either demosaicing or JPEG compression (but not both) should remain achiral. (2) On the other hand, surprisingly, random crops (avoiding a 16-pixel margin around the boundary in the cropped image) on that achiral distribution of images after both demosaicing and JPEG compression may likely become chiral.

To verify this analysis empirically, we again train ResNet models on the same achiral Gaussian distributions as introduced in Section~\ref{sec:nocropsize}. Specifically, we take random crops of size (512, 512) from the center (544, 544) of the (576,576) Gaussian noise images to avoid possible boundary effects from a 16-pixel margin. We train separate networks on each of the three output image distributions obtained from applying each of the three imaging operations (demosaicing, JPEG compression, and composition of demosaicing followed by JPEG compression) on the initial Gaussian noise image distribution. 
Note that, as before, we perform a log-scale grid search over learning rates.

The network training results show that neither demosaicing nor JPEG compression alone is sufficient to produce a chiral distribution under random cropping: models trained with such images fail to achieve more than 50\% accuracy. This suggests that chirality is preserved when those operations are applied in isolation.
But, as our theory predics, when both operations are applied the image distribution becomes chiral: the trained network achieves 100\% training and test accuracy. This supports our theoretical analysis of the glide-commutativity. Together, our analysis and empirical study suggest that chiral traces are left in photographs via the Bayer demosaicing and JPEG compression imaging processes.

\section{Conclusion}\label{sec:conclusion}
In this document we have developed theory relating the preservation of symmetry by various operations to their commutativity with corresponding symmetry transformations. We proposed the commutative residual as a tool for analyzing symmetry preservation, and predicting how different operations will affect the results of deep learning. We also extend our theory to random cropping and show how to evaluate glide commutativity to detect permuted commutativity. Our theoretical analysis and empirical experiment suggest that when demosaicing and JPEG compression are applied together, achiral distributions can becomes chiral, which has implications on several areas, including self-supervised learning, image forensics, data augmentation.

{\small
\bibliographystyle{supp/ieee}
\bibliography{supplemental}
}